\documentclass[twoside]{article}

\usepackage[accepted]{aistats2023}
\usepackage[round]{natbib}

\usepackage{graphicx,amsmath,amsthm,booktabs,subfig,amssymb,multirow,float,array,cuted,url,soul,xcolor,pgfpages,graphicx,caption,color,afterpage,tcolorbox,bm,adjustbox}
\usepackage[ruled,vlined,linesnumbered]{algorithm2e}
\usepackage[breaklinks=true, colorlinks, linkcolor = blue, urlcolor  = blue, citecolor = blue, anchorcolor = blue, bookmarks=true]{hyperref} 
\usepackage[switch]{lineno}
\allowdisplaybreaks
\newcommand{\R}{\mathbb{R}}
\renewcommand{\L}{\mathcal{L}}

\newcommand{\abs}[1]{\left\lvert#1\right\rvert}
\newcommand{\norm}[1]{\left\lvert\left\lvert#1\right\rvert\right\rvert}
\newcommand{\dsp}{\displaystyle}

\newcolumntype{P}[1]{>{\centering\arraybackslash}p{#1}}

\newtheorem{theorem}{Theorem}

\newtheorem{corollary}{Corollary}
\numberwithin{lemma}{subsection}

\begin{document}

%

%

\twocolumn[

\aistatstitle{Orthogonal Gated Recurrent Unit with Neumann-Cayley Transformation}

\aistatsauthor{Edison~Mucllari$^{*}$ \And Vasily~Zadorozhnyy$^{*}$ \And Cole Pospisil \And Duc~Nguyen \And Qiang~Ye}

\aistatsaddress{Department of Mathematics, University of Kentucky, Lexington, Kentucky, USA\\
\texttt{\{edison.mucllari, vasily.zadorozhnyy, Cole.Pospisil, ducnguyen, qye3\}@uky.edu}}
]

\begin{abstract}
    In recent years, using orthogonal matrices has been shown to be a promising approach in improving Recurrent Neural Networks (RNNs) with training, stability, and convergence, particularly, to control gradients. While Gated Recurrent Unit (GRU) and Long Short Term Memory (LSTM) architectures address the vanishing gradient problem by using a variety of gates and memory cells, they are still prone to the exploding gradient problem. In this work, we analyze the gradients in GRU and propose the usage of orthogonal matrices to prevent exploding gradient problems and enhance long-term memory. We study where to use orthogonal matrices and we propose a Neumann series-based Scaled Cayley transformation for training orthogonal matrices in GRU, which we call Neumann-Cayley Orthogonal GRU, or simply NC-GRU. We present detailed experiments of our model on several synthetic and real-world tasks, which show that NC-GRU significantly outperforms GRU as well as several other RNNs.
\end{abstract}

\section{Introduction}\label{introduction}

One of the preferred neural network models for working with sequential data is the Recurrent Neural Network (RNN) ~\citep{10.5555/104279.104293, doi:10.1073/pnas.79.8.2554}. RNNs can efficiently model sequential data through a hidden sequence of states. However, training vanilla RNNs have obstacles~\citep{10.5555/104279.104293, doi:10.1073/pnas.79.8.2554}, one of which is their susceptibility to vanishing and exploding gradients~\citep{Bengio93}. In the case of vanishing gradients, the optimization algorithm faces difficulty continuing to learn due to very small changes in the model parameters. In the case of exploding gradients, the training could overstep local minima, potentially causing instabilities such as divergence or oscillatory behavior. It may also lead to computational overflows.

There have been several works studying how to solve these problems. For example gates have been introduced into the RNN architecture: Long Short-Term Memory (LSTM)~\citep{Hochreiter:1997:LSM:1246443.1246450} and Gated Recurrent Units (GRU)~\citep{https://doi.org/10.48550/arxiv.1406.1078}. They can pass long term information and help to overcome vanishing gradients. In practice, GRU and LSTM models are still prone to the problem of exploding gradients.

More recently, several RNN models have been proposed using unitary or orthogonal matrices for the recurrent weights~\citep{kyle17, DoradoRojas2020OrthogonalLR, vorontsov2017orthogonality, Mhammadi16, 10.5555/3305381.3305560, 10.5555/3045390.3045509, Wisdom16, 8668730, https://doi.org/10.48550/arxiv.1811.04142} along with methods to preserve those properties. The introduction of such weights into RNN models brought new development into the RNN community. One of the main reasons behind this is a theoretical explanation of why the performance is better when using unitary or orthogonal weights~\citep{10.5555/3045390.3045509}. The key step in all of these methods is preserving orthogonal or unitary properties at every iteration of training. There have been several different techniques for updating the recurrent weights to preserve either orthogonal or unitary properties, including for example multiplicative updates~\citep{Wisdom16}, Givens rotations~\citep{10.5555/3305381.3305560}, Householder reflections~\citep{Mhammadi16}, Cayley transforms~\citep{kyle17, https://doi.org/10.48550/arxiv.1811.04142, https://doi.org/10.48550/arxiv.1911.07964,pmlr-v97-lezcano-casado19a}, and other similar ideas that have shown effective usage of orthogonal or unitary matrices~\citep{Saxe14, 10.5555/3045390.3045509, Henaff17, Hyland17, tagare_notes_2011, vorontsov17}. 

In this work, we study the benefits of applying orthogonal matrices to one of the most widely used RNN models, Gated Recurrent Unit (GRU)~\citep{https://doi.org/10.48550/arxiv.1406.1078}, both theoretically and experimentally. We analyze the gradients of the GRU loss and based on this analysis, we propose the usage of orthogonal matrices in several hidden state weights of the model. We introduce a Neumann series-based Scaled Cayley transforms for training the orthogonal weights. Our method utilizes a reliable and stable method of Scaled Cayley transforms which was studied and used in~\citep{kyle17, https://doi.org/10.48550/arxiv.1811.04142} for training orthogonal weights for RNNs. In addition, we propose a Neumann series approximation of matrix inverse inside the Cayley transform. Such substitution not only performs either on the same or even better level as the traditional inverse (see Section \ref{experiments} for experiments), but also decreases computation time which might come of particular assistance when working with larger models. We call our method \emph{Neumann-Cayley Orthogonal Gated Recurrent Unit} or \emph{NC-GRU} for simplicity. Our experiments show that the proposed method is more stable with faster convergence and produces better results on several synthetic and real-world tasks.

The rest of the paper has the following layout. In Section \ref{relatedwork}, we discuss some previous works that are most relevant for this paper. In Section \ref{theory}, we introduce Neumann-Cayley Orthogonal Gated Recurrent Unit (NC-GRU) method together with backpropagation method and gradient analysis of GRU~\citep{https://doi.org/10.48550/arxiv.1406.1078} and NC-GRU. In Section \ref{experiments}, we present experiments of our model on several synthetic and real-world problems: Adding task~\citep{hochreiter1997long}, Copying task~\citep{hochreiter1997long}, Parenthesis task~\citep{8668730, https://doi.org/10.48550/arxiv.1611.09434}, and Denoise task~\citep{8668730} as well as character and word level Penn TreeBank prediction tasks~\citep{marcus-etal-1993-building}. Then, in Section \ref{ablation_study}, we provide several experiments to justify the use of orthogonal matrices, and Neumann series as well as which hidden states benefit from them. Finally, in Section \ref{conclusion}, we summarize our proposed method and contribution.

\section{Related Work}\label{relatedwork}

Many models have been designed to improve classical RNN~\citep{10.5555/104279.104293, doi:10.1073/pnas.79.8.2554} for sequential data. They include, for example, the establishment of gates~\citep{Hochreiter:1997:LSM:1246443.1246450,https://doi.org/10.48550/arxiv.1406.1078}, normalization methods~\citep{ioffe2015batch, CooijmansBLGC17, Wu_2018_ECCV, ulyanov2017instance, NIPS2016_ed265bc9, ba2016layer, NEURIPS2019_2f4fe03d}, and the introduction of unitary and orthogonal matrices into RNN structure~\citep{10.5555/3045390.3045509,10.5555/3305381.3305560,Mhammadi16,vorontsov17, 8668730, DoradoRojas2020OrthogonalLR}, etc. In this section, we will discuss some of the works most relevant to our proposed method.

Unitary RNNs (uRNNs)~\citep{10.5555/3045390.3045509} presented an architecture that learns a unitary weight matrix. The construction of the recurrent weight matrix consists of a composition of diagonal matrices, reflection matrices in the complex domain, and Fourier transform. The uRNN model presented in~\citep{Wisdom16} is based on the constrained optimization over the space of all unitary matrices rather than a product of parameterized matrices. EUNN~\citep{10.5555/3305381.3305560} is another work that utilizes the product of unitary matrices. The recurrent matrix in this architecture is parameterized with products of Givens Rotations. Also, the representation capacity of the unitary space is fully tunable and ranges from a subspace of unitary matrices to the entire unitary space.

Work by~\citep{Mhammadi16} proposed orthogonal RNNs, or simply oRNNs, which involve the application of Householder reflections. Such parameterization of the transition matrix allows efficient training and maintains the orthogonality of the recurrent weights while training. The method introduced in~\citep{vorontsov17}, proposes a weight matrix factorization by bounding the matrix norms. Moreover, it controls the degree of gradient expansion during backpropagation. Besides that, this technique enforces orthogonality as well.
GORU~\citep{8668730} presented a RNN which is an extension of unitary RNN with a gating mechanism. Unitarity is preserved by using the ideas from EURNN~\citep{10.5555/3305381.3305560}. The results compared to GRU~\citep{https://doi.org/10.48550/arxiv.1406.1078} are mixed, depending on the task. More recently, ~\citep{DoradoRojas2020OrthogonalLR} presented an embedding of a linear time - invariant system that contains Laguerre polynomials in the model.

\subsection{Gated Recurrent Unit (GRU)}
In our work, we have studied in depth the Gated Recurrent Unit (GRU) architecture which was proposed in~\citep{https://doi.org/10.48550/arxiv.1406.1078} as an alternative to the well-known LSTM~\citep{Hochreiter:1997:LSM:1246443.1246450} cell. The structure of a GRU cell is
\begin{equation}
    \begin{aligned}\label{model:GRU}
        r_t &= \sigma \left(W_rx_t + U_r h_{t-1} + b_r\right)\\
        u_t &= \sigma \left(W_ux_t + U_u h_{t-1} + b_u\right)\\
        c_t &= \Phi \left(W_cx_t + U_c \left(r_t \odot h_{t-1}\right) + b_c\right)\\
        h_t &= \left(1 - u_t\right) \odot h_{t-1} + u_t \odot c_t
    \end{aligned}
\end{equation}
where $W_r$, $W_u$, and $W_c$ are input weights in $\R^{n \times m}$, $U_r$, $U_u$, and $U_c$ are recurrent weights in $\R^{n \times n}$, and $b_r$, $b_u$ and $b_c$ are the bias parameters in $\R^n$. Here, $m$ represents the dimension of the input data and $n$ represents the dimension of the hidden state. In (\ref{model:GRU}), the activation functions $\sigma$ and $\Phi$ are sigmoid and hyperbolic tangent function ($\tanh$) respectively, and $\odot$ is the Hadamard product. In addition, the initial hidden state $h_0$ is set to zero.

The main difference of GRU from LSTM is the implementation of the long-term memory not as a separate channel but inside the hidden state $h_t$ itself. GRU has a single gate $u_t$ that controls both forget and input gates. For example, if the output of $u_t$ is $1$ then the forget gate is open implying that the input gate is closed. Similarly, if $u_t$ is $0$, the forget gate is closed and input gate is open. This structure allows GRUs to discard random or insignificant information but at the same time grasping the important details.

\subsection{Cayley Transform Orthogonal RNN}\label{scoRNN}
We have implemented and studied the effects of orthogonal matrices inside the GRU cell based on the Cayley transforms~\citep{tagare_notes_2011}. Some initial work to use Cayley transform for orthogonal weights in RNNs was presented in~\citep{kyle17} together with scoRNN model. This model introduced a skew-symmetric matrix $A$ which is used to define an orthogonal matrix $W$ via the Scaled Cayley transform
\begin{equation}
    W = \left(I+A \right)^{-1}\left(I-A \right) D,
\end{equation}
where matrix $D$ is a diagonal matrix of ones and negative ones, which scales the traditional Cayley transform~\citep{tagare_notes_2011}. It is proved in \citep{KAHAN2006335} that the matrix $D$, with a suitable choice on the number of negative ones, can avoid a potential problem of eigenvalue(s) of $A$ being negative one, making matrix $I+A$ non-invertible. The number of negative ones in matrix $D$ can be considered as a tunable hyperparameter. Further, it guarantees that the skew-symmetric matrix $A$ that generates the orthogonal matrix will be bounded.

\citep{kyle17} presents the following process to train scoRNN model using Scaled Cayley transforms:

\begin{align}
    A^{(k+1)} & = A^{(k)} - \lambda \nabla_{A}L\left(U_{sco}\left(A^{(k)}\right)\right)\\
    U_{sco}^{(k+1)} & = \left(I+A^{(k+1)} \right)^{-1}\left(I-A^{(k+1)} \right) D \label{sct:u}
\end{align}
where $\nabla_{A}L\left(U_{sco}\left(A\right)\right)$ is computed using 
\begin{equation}
    \nabla_{A}L\left(U_{sco}\left(A\right)\right) = V^T-V,\, 
\end{equation}
with
\begin{equation}
    V= \left(I+A\right)^{-T}\nabla_{U_{sco}} L\left(U_{sco}\left(A\right)\right) \left(D+U_{sco}^T \right)
\end{equation}
in which $\nabla_{U_{sco}} L\left(U_{sco}\left(A\right)\right)$ is computed via standard backpropogation methods.

Despite the fact that rounding errors may accumulate over a number of repeated matrix multiplications, orthogonality in scoRNN~\citep{kyle17} is maintained to the machine precision. This property helps to achieve significant improvements over other orthogonal/unitary RNNs for long sequences on several benchmark tasks, see the Experiments section in~\citep{kyle17} for more details.

\section{Efficient Orthogonal Gated Recurrent Unit}\label{theory}

We now present an efficient orthogonal GRU. The proofs of all theoretical results presented in this section are given in section \ref{sm:sec}.

\subsection{Gradient Analysis of Hidden States in GRU}
Gradient behavior plays a very important role in model training, convergence, stability, and most importantly performance. However when it comes to backpropagation through time for the GRU model from (\ref{model:GRU}), the gradients of the loss function $\mathcal{L}$ with respect to intermediate hidden states, weights, and biases can be found from the respective gradients of the final hidden state $h_T$, which is simplified to finding the gradient of $h_t$ with respect to $h_{t-1}$ for $t$ between 1 and $T$. Namely, 
\begin{equation}
    \dfrac{\partial {\cal L}}{\partial h_{i}}=\dfrac{\partial {\cal L}}{\partial h_{T}}\prod_{t=i+1}^{T} \dfrac{\partial h_{t}}{\partial h_{t-1}}.
\end{equation}

Thus, to analyze the gradients, we consider the gradient of the hidden state $h_t$ with respect to the hidden state $h_{t-1}$ as well as its upper bound in the following theorem.

\begin{theorem}\label{thm:1}
    Let $h_{t-1}$ and $h_{t}$ be two consecutive hidden states from the GRU model stated in (\ref{model:GRU}). Then
    \begin{equation}
        \norm{\dfrac{\partial h_{t}}{\partial h_{t-1}}}_2 \leq \alpha + \beta\norm{U_c}_2
    \end{equation}
    where
    \begin{equation}
        \begin{aligned}
            \alpha &= \delta_u\left( \max_i\left\{[h_{t-1}]_i\right\} + \max_i\left\{[c_t]_i\right\}\right)\norm{U_u}_2 \\
            &\quad\quad\quad + \max_i\left\{(1-[u_t]_i)\right\}
        \end{aligned}
    \end{equation}
    and
    \begin{equation}
        \begin{aligned}
            \beta &= \max_i\left\{[u_t]_i\right\}\left(\delta_r \norm{U_r}_2 \max_i \left\{[h_{t-1}]_i\right\}\right.\\
            &\quad\quad\quad \left.+ \max_i \left\{[r_t]_i\right\} \right),
        \end{aligned}
    \end{equation}
    with constants $\delta_u$ and $\delta_r$ defined as follows:
    \begin{equation}
        \delta_u = \max_i \left\{\left[u_t\right]_i \left(1-\left[u_t\right]_i\right)\right\}
    \end{equation}
    and
    \begin{equation}
        \delta_r = \max_i \left\{\left[r_t\right]_i \left(1-\left[r_t\right]_i\right)\right\}.
    \end{equation} 
\end{theorem}

The following corollary provides some simple upper bounds for $\alpha, \beta$ obtained in the above theorem.

\begin{corollary}\label{cor:post_thm_bounds}
    For the hyperbolic tangent activation function in (\ref{model:GRU}) (i.e. $\Phi=\,$\verb|tanh|), we have $\delta_u, \delta_r \le \frac{1}{4}$, $[h_t]_i \le 1$ for any $i$ and $t$ as well as 
    \begin{equation}
        \alpha \le \dfrac{1}{2}\norm{U_u}_2+1 \quad\text{and}\quad \beta\leq \dfrac{1}{4}\norm{U_r}_2+1.
    \end{equation}
\end{corollary}

These bounds may be pessimistic because the gate elements may be expected to be close to 0 or 1. 
As a consequence, the following corollary presents the relationship between the constants $\alpha$ and $\beta$ presented in Theorem \ref{thm:1} when GRU's gates are nearly closed or opened. Below we use a notation $x \lesssim y$ to denote that $x$ is bounded by a quantity that is approximately equal to $y$.

\begin{corollary}\label{cor:1}
    When elements of GRU gates $u_t$ and $r_t$ are nearly either 0 or 1, then constants $\alpha$ and $\beta$ from Theorem \ref{thm:1} satisfy the following inequality:
    \begin{equation}
        \alpha+\beta\lesssim 2.
    \end{equation}
    Moreover if $u_t$ and $r_t$ are nearly either the zero vector or the vector of all ones, then 
    \begin{equation}
        \alpha+\beta\lesssim 1.
    \end{equation}
\end{corollary}

\subsection{Neumann-Cayley Orthogonal Transformation} \label{sec:backprop}

Based on Theorem \ref{thm:1} and Corollary \ref{cor:1}, we propose the usage of orthogonal weights in the hidden parameters of GRU to obtain a better conditioned gradients. As discussed in section \ref{relatedwork}, there have been different techniques proposed and used to initialize and preserve orthogonal weights while training, for example Givens rotations~\citep{10.5555/3305381.3305560}, Householder reflections~\citep{Mhammadi16} etc. In this work, we implement a version of the Scaled-Cayley transformation method discussed in \ref{scoRNN} with one key difference. The Scaled Cayley transform method requires calculation of the inverse of $I+A^{(k)}$ to update the orthogonal matrix $U^{(k)}$ in (\ref{sct:u}). When it comes to the computation of this inverse, classical numerical methods such as using LU-decomposition or solving the Least Squares problem can be implemented, and these methods work well when the dimension of the matrix is small. However, if the dimension of the matrix is large then these methods are very expensive from both memory and computational time perspectives. Moreover, classical methods might overflow and not converge at all. We propose to solve this possible complication by using the Neumann Series method to approximate the inverse of $I+A^{(k)}$. 

To derive the Neumann series approximation for the inverse of $I+A^{(k)}$ in (\ref{sct:u}) we consider the following:
\begin{align}
    \left(I+A^{(k)}\right)^{-1} & = \left(I+A^{(k-1)}-\delta A^{(k)}\right)^{-1}\label{eq:neumann_series1}\\
    &\hspace{-0.9in}\adjustbox{width=0.9\linewidth}{$=\left(I-\left(I+A^{(k-1)}\right)^{-1}\delta A^{(k)}\right)^{-1}\left(I+A^{(k-1)}\right)^{-1}$}\label{eq:neumann_series2}\\
    &\hspace{-0.9in}\adjustbox{width=0.9\linewidth}{$=\left(\dsp\sum_{i=0}^{\infty}\left(\left(I+A^{(k-1)}\right)^{-1}{\delta A^{(k)}}\right)^i\right)\left(I+A^{(k-1)}\right)^{-1}$}\label{eq:neumann_series3}
\end{align}
where $\delta A^{(k)} := optimizer_A\left(\nabla_A\L={V^{(k)}}^T-V^{(k)}\right)$, here $optimizer_A$ includes a learning rate $\lambda$ inside of it. Note that the equality between Equations (\ref{eq:neumann_series2}) and (\ref{eq:neumann_series3}) relies on the assumption that $\norm{\left(I+A^{(k-1)}\right)^{-1}\delta A^{(k)}}<1$ for some operator norm $\norm{\cdot}$, see~\citep{demmel97} for more details. We have conducted an ablation study that shows empirical evidence that this condition is indeed satisfied, see section \ref{as:norms} for more details.

In our experiments, we have considered the first and the second-order Neumann series approximations, estimating the series in (\ref{eq:neumann_series3}) with two ($i=0,1$) and three ($i=0,1,2$) terms respectively. As expected, the performance of the model is slightly better when using second-order approximation although it comes with a marginal increase in computational time, see \ref{as:inv_vs_neumann} for the ablation study regarding the accuracy and computational time of such approximations. Mathematically speaking, if we are using the second-order approximation, the error is of order $\mathcal{O}\left(\left(\left(I+A^{(k-1)}\right)^{-1}\delta A^{(k)}\right)^{3}\right)$. Even though this error is quite small, there is a chance that the errors from this approximation can accumulate and cause a loss of orthogonality. To avoid this issue we recommend resetting orthogonality by computing the matrix inverse explicitly using a factorization method at the beginning of each epoch. However, it might be necessary to do it more often particularly in the earlier training (e.g. every 100 iterations) due to more fluctuations in the gradients.

\IncMargin{1em}\SetNlSty{text}{}{:}
\begin{algorithm}
    \footnotesize
    \SetAlgoLined
        \textbf{Given:} $D$, $A^{(0)}$, $U^{(0)}$, $\nabla_{U}\L\left(U^{(0)} \left(A^{(0)}\right)\right)$, $optimizer_{A}$
        
        \textbf{Define:} $\tilde{A}^{(0)}:= \left(I+A^{(0)}\right)^{-1}$
        
        \vspace{2pt}
        
        \For{$k=1,2,\ldots$}
        {
            $V^{(k)} := \text{$\tilde{A}^{(k-1)}$}^T \nabla_{U}\L\left(U^{(k-1)} \left( A^{(k-1)}\right)\right)\left(D+{U^{(k-1)}}^T \right)$
            
            $\delta A^{(k)} := optimizer_{A}\left(\nabla_{A}\L={V^{(k)}}^{T}-V^{(k)} \right)$
            
            $A^{(k)} := A^{(k-1)}-\delta A^{(k)}$
            
            $\tilde{A}^{(k)} := \left(I+\tilde{A}^{(k-1)}{\delta A^{(k)}}+\left(\tilde{A}^{(k-1)}{\delta A^{(k)}}\right)^2\right)\tilde{A}^{(k-1)}$
            
            $U^{(k)} := \tilde{A}^{(k)} \left(I-A^{(k)} \right) D$
            
        }
    \caption{Update Rule for Orthogonal Weight $U$} 
    \label{alg:ncgru}
\end{algorithm}

We present Algorithm \ref{alg:ncgru} that outlines the Neumann-Cayley Orthogonal Transformation method for training weight $A$ as well as updating the corresponding orthogonal weight $U$. It is important to note that during the initialization step, the weight $A^{(0)}$ is defined to be skew-symmetric using the same initialization technique as in~\citep{kyle17} which is based on the idea from ~\citep{Henaff17}. Then, we apply the Cayley transform in $A^{(0)}$ to obtain $U^{(0)}$. Another peculiar detail that we want to point out is that Algorithm \ref{alg:ncgru} includes $optimizer_A$, which is the standard optimizer such as SGD, RMSProp~\citep{tieleman2012lecture}, or Adam~\citep{kingma2014adam}, etc., that takes $\nabla_A\L={V^{(k)}}^T-V^{(k)}$ as an input. Moreover, the skew-symmetric property of the weight $A$ and its gradient are preserved under the such optimizer.

\subsection{Neumann-Cayley Orthogonal GRU (NC-GRU)}\label{ncgru}

Finally, we introduce a Neumann-Cayley Orthogonal GRU (NC-GRU) model that utilizes the proposed Neumann-Cayley Orthogonal Transform.

The structure of a NC-GRU cell is shown below:

\begin{equation}
    \begin{aligned}\label{model:NC-GRU}
        r_t &= \sigma \left(W_rx_t + U_r(A_r) h_{t-1} + b_r\right)\\
        u_t &= \sigma \left(W_ux_t + U_u h_{t-1} + b_u\right)\\
        c_t &= \Phi \left(W_cx_t + U_{c}(A_c) \left(r_t \odot h_{t-1}\right)\right)\\
        h_t &= \left(1 - u_t\right) \odot h_{t-1} + u_t \odot c_t
    \end{aligned}
\end{equation}
here $\sigma$ represents sigmoid function, $\odot$ - Hadamard product, and $\Phi$ - modReLU function defined in~\citep{10.5555/3045390.3045509} as
\begin{equation}\label{eq:moderelu}
    \Phi(x):=\text{modReLU}(x):=\text{sgn}(x)\cdot\text{ReLU}\left(\abs{x} + b \right)
\end{equation}
with $b$ as a trainable bias. 

We derive from most of the experiments that the best performance is achieved using orthogonality in $U_c$ and $U_r$ hidden weights. In addition, we present an ablation study in \ref{as:ortho_in_Uu_Ur} about the usage and performance of orthogonal weights throughout the GRU model. Similar to the GRU Cell, $W_r$, $W_u$, $W_c$, $U_u$, $b_r$, $b_u$, $b$, are trainable parameters as well as $U_r$ and $U_c$ together with their associated weights $A_r$ and $A_c$ respectively. Moreover, all of them except $U_{r}$ with $A_r$ and $U_c$ with $A_c$ are trained using standard backpropagation algorithms such as Stochastic Gradient Descent (SGD), RMSProp~\citep{tieleman2012lecture}, or Adam~\citep{kingma2014adam} in a similar manner as in GRU~\citep{https://doi.org/10.48550/arxiv.1406.1078}, however weights $U_{r}$, $A_r$, $U_c$, and $A_c$ are trained using Algorithm \ref{alg:ncgru}.

As we mentioned before, orthogonal weights lead to a better conditioned gradient, the following Corollary summarizes this result.

\begin{corollary}\label{cor:2}
    Let $h_{t-1}$ and $h_{t}$ be two consecutive hidden states from the NC-GRU model defined in (\ref{model:NC-GRU}). Then $\norm{U_r}_2=\norm{U_c}_2=1$ and if the element of the gates $u_t$ and $r_t$ are nearly either 0 or 1, then the following inequality is satisfied: 
    \begin{equation}
        \norm{\dfrac{\partial h_{t}}{\partial h_{t-1}}}_2 \lesssim 2.
    \end{equation}
    Furthermore, if $u_t$ and $r_t$ are nearly either zero vector or vector of all ones,
    \begin{equation}
        \norm{\dfrac{\partial h_{t}}{\partial h_{t-1}}}_2 \lesssim 1.
    \end{equation}
\end{corollary}

\begin{figure*}[!t]
    \centering
    \subfloat[$T=100$]{{\includegraphics[width=0.45\textwidth]{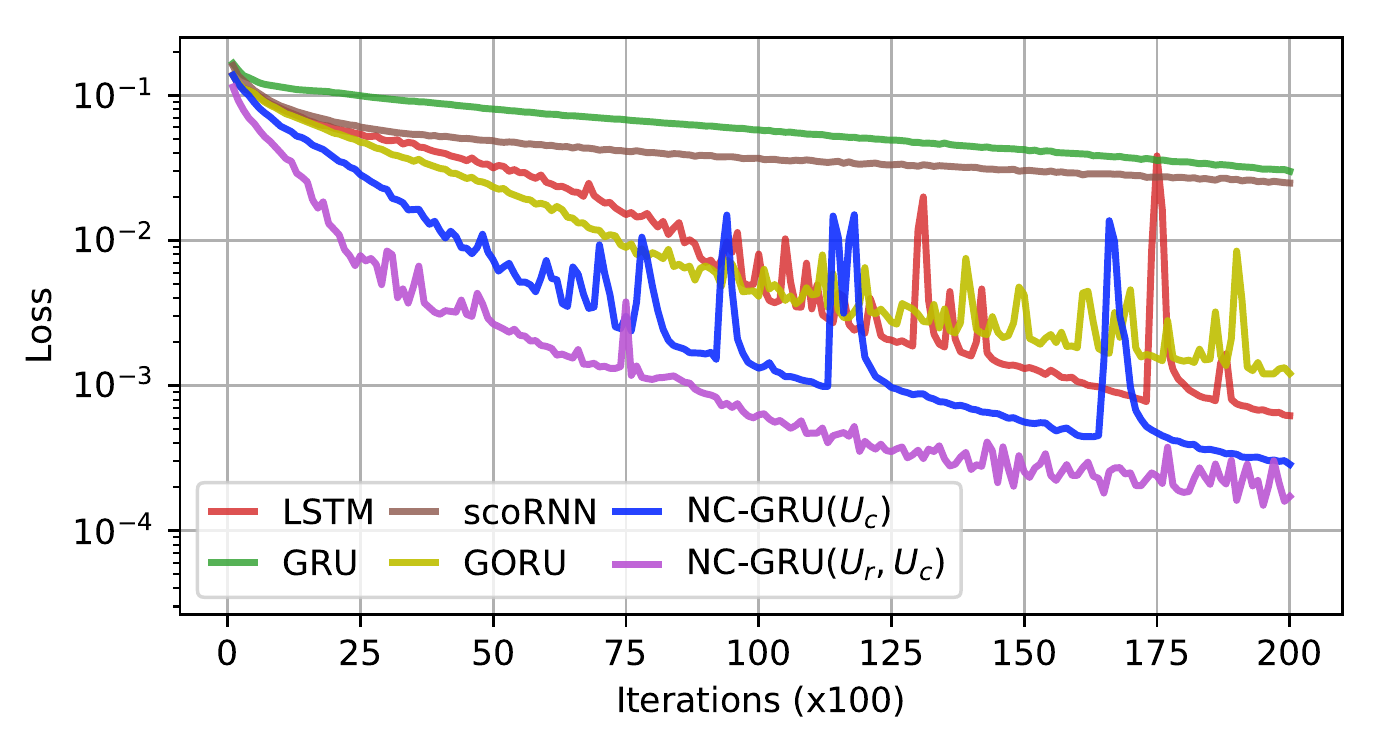} }\label{fig:parenthesis_task100}}
    \qquad
    \subfloat[$T=200$]{{\includegraphics[width=0.45\textwidth]{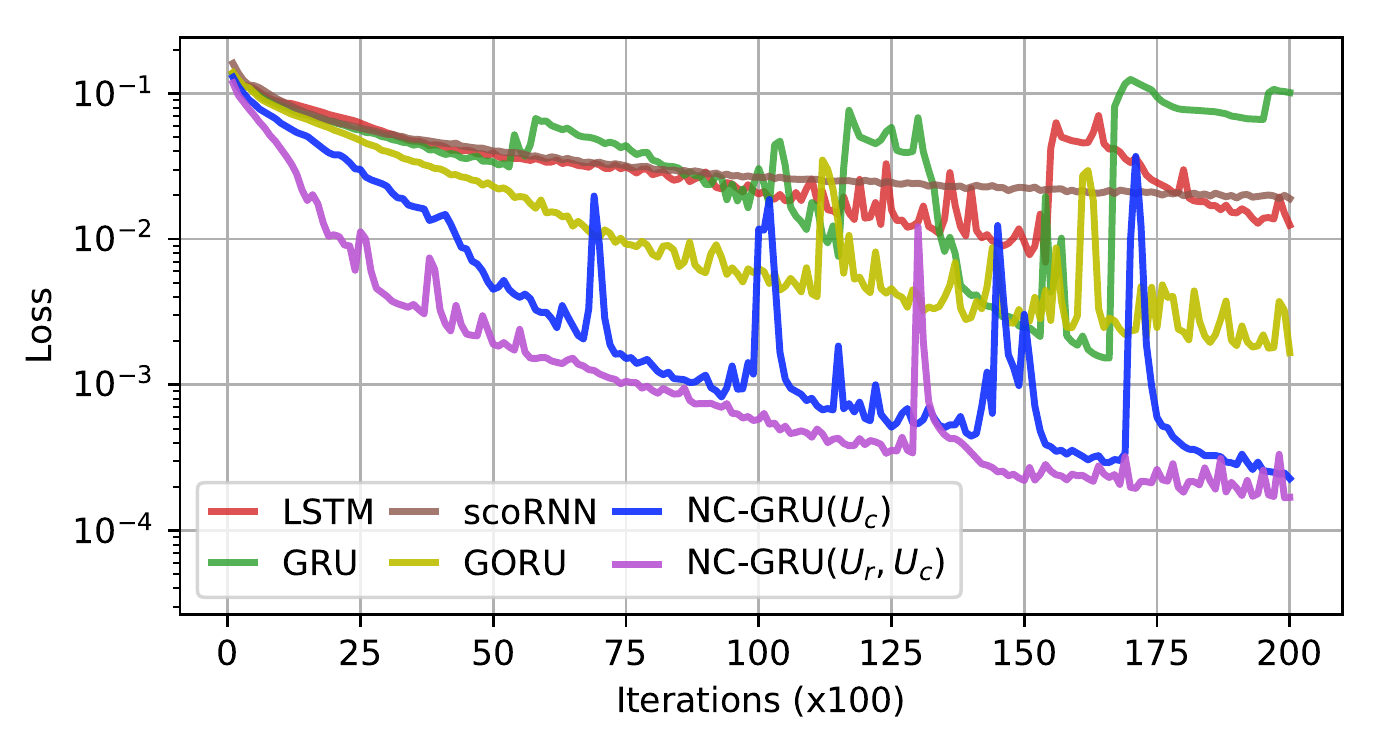} }\label{fig:parenthesis_task200}}
    \caption{\textbf{Parenthesis Task Results}}
    \label{fig:parenthesis_task}
\end{figure*}

\section{Experiments}\label{experiments}
We have performed a variety of experiments to demonstrate the robustness and efficiency of our proposed NC-GRU method. To this end, we apply NC-GRU to 4 commonly used synthetic tasks: Parenthesis, Denoise, Adding, and Copying Tasks. In addition, we have considered non-synthetic experiments, Language Modeling with the character and word level tasks for the Penn Treebank (PTB) dataset~\citep{10.5555/972470.972475}.

All the experiments were trained with equal numbers of trainable parameters also known as parameter-matching architecture, see the implementation details section of each experiment for more details. Codes for these experiments are available online at \href{https://github.com/Edison1994/NC-GRU}{github.com/Edison1994/NC-GRU}.

\subsection{Parenthesis Task}\label{parenthesisproblem}

This experiment derives from the descriptions in~\citep{8668730} and \citep{https://doi.org/10.48550/arxiv.1611.09434}. This task tests the ability of the network to remember the number of unmatched parentheses contained in our input data. The input data consists of 10 pairs of different types of parentheses combined with some noise data in between and it is given as a one-hot encoding vector of length $T$. As stated in~\citep{8668730}, there are not more than $10$ types of parentheses. The output data is given as a one hot-encoding vector too, counting the number of the unpaired parentheses in the corresponding input data. The goal of our model is to forget the noise data and absorb information from the long-term dependencies related to the parentheses. This synthetic data requires the model to develop a memory and to be able to select the most relevant information.

For this experiment, we compared 5 different models/architectures: LSTM~\citep{Hochreiter:1997:LSM:1246443.1246450}, GRU~\citep{https://doi.org/10.48550/arxiv.1406.1078}, scoRNN~\citep{kyle17}, GORU~\citep{8668730}, and NC-GRU (ours).

\emph{Implementation Details:} All of the models consisted of a single layer net with the following hidden dimensions for each model: LSTM~\citep{Hochreiter:1997:LSM:1246443.1246450} - 42, GRU~\citep{https://doi.org/10.48550/arxiv.1406.1078} - 50, scoRNN~\citep{kyle17} - 110, GORU~\citep{8668730} - 64, and NC-GRU - 56. Furthermore, we trained all of the models for 200 epochs with a batch size of 16, and the number of negative ones for the $D$ matrix in scoRNN~\citep{kyle17} and NC-GRU models were 20 and 40, respectively. All models were trained using Adam optimizer~\citep{kingma2014adam} with a learning rate of $10^{-3}$ including $A$ associated weights in scoRNN~\citep{kyle17} and NC-GRU models. We conducted experiments using input length of $T=100$, see Figure \ref{fig:parenthesis_task100}, and $T=200$, see Figure \ref{fig:parenthesis_task200}. 

We present two versions of NC-GRU model, the first one only has orthogonality in $U_c$ weight (NC-GRU($U_c$)) however the second model utilizes orthogonality in both weights $U_r$ and $U_c$ (NC-GRU($U_r,U_c$)). Both models were trained using Neumann series method with a reset at every 50 iterations.

\emph{Results:} We observed and recorded the behavior of the five models on the Parenthesis task when the input length is set to $100$ and $200$. Our results showed that both versions of NC-GRU models outperform GRU~\citep{https://doi.org/10.48550/arxiv.1406.1078}, LSTM~\citep{Hochreiter:1997:LSM:1246443.1246450}, scoRNN~\citep{kyle17}, and GORU~\citep{8668730} models with a significant gap, see Figure \ref{fig:parenthesis_task}. On this task, NC-GRU($U_r,U_c$) model shows a better performance than NC-GRU($U_c$), however both of them perform better than the rest of the compared models. The minimum value of the loss function attained during the training are presented in Table \ref{tab:parenthesis_task}.

\begin{table}[ht]
    \centering
    \caption{\textbf{Parenthesis Task Results:} Minimum attained loss values (the smaller the better). All results are based on our tests. NC-GRU$(U_c)$ denotes the NC-GRU model (\ref{model:NC-GRU}) where Neumann-Cayley method was only applied to the weight $U_c$, similarly NC-GRU$(U_r, U_c)$ represents the NC-GRU model (\ref{model:NC-GRU}) where both $U_r$ and $U_c$ weights were updated using the Neumann-Cayley method.}
    \begin{tabular}{r|P{1.6cm}|P{1.6cm}}
        \toprule
        & \multicolumn{2}{c}{Loss $\times 10^{-3}$ $\downarrow$} \\
        \midrule
        Sequence Length & $T=100$ & $T=200$\\
        \midrule
        LSTM & $0.62$ & $6.93$\\
        GRU & $29.94$ & $1.526$\\
        scoRNN & $24.85$ & $18.957$\\
        GORU & $1.199$ & $1.66$ \\
        NC-GRU$(U_c)$ (\textbf{ours}) & $0.286$ & $0.227$ \\
        NC-GRU$(U_r, U_c)$ (\textbf{ours}) & $0.149$ & $0.167$ \\
        \bottomrule
    \end{tabular}
    
    \label{tab:parenthesis_task}
\end{table}

\subsection{Denoise Task}\label{denoiseproblem}

The Denoise Task~\citep{8668730} is another synthetic problem that requires filtering out the noise out of a noisy sequence. This problem requires the forgetting ability of the network as well as learning long-term dependencies coming from the data~\citep{8668730}. The input sequence of length $T$ contains 10 randomly located data points and the other $T-10$ points are considered noise data. These 10 points are selected from a dictionary $\{a_i\}_{i=0}^{n+1}$, where the first $n$ elements are data points, and the other two are the $``noise"$ and the $``marker"$ respectively. The output data consists of the list of the data points from the input, and it should be outputted as soon as it receives the $``marker"$. The model task to filter out the noise part and output the random 10 data points chosen from the input. 

Similar to the Parenthesis Task, we compared our NC-GRU models to LSTM~\citep{Hochreiter:1997:LSM:1246443.1246450}, GRU~\citep{https://doi.org/10.48550/arxiv.1406.1078}, scoRNN~\citep{kyle17}, and GORU~\citep{8668730} models.

\emph{Implementation Details:} We implemented one NC-GRU cell with a hidden size of 118 and the number of negative ones in the $D$ matrix to 50. The hidden size for the LSTM~\citep{Hochreiter:1997:LSM:1246443.1246450} was 90, GRU~\citep{https://doi.org/10.48550/arxiv.1406.1078} -- 100, scoRNN~\citep{kyle17} -- 200, and GORU~\citep{8668730} -- 128. We implemented Adam optimizer~\citep{kingma2014adam} with learning rate of $10^{-3}$ to train all the aforementioned models including scoRNN~\citep{kyle17} and NC-GRU $A$ weights. We trained all the models for 10,000 iterations with a batch size of 128. Similar to the Parethesis Task \ref{parenthesisproblem}, we implemented the Neumann series method of approximation the $\left(I+A^{(k)}\right)^{-1}$ when training the NC-GRU model with the reset option to be applied every 50 iterations.

\emph{Results:} Based on our experiments, NC-GRU($U_c$) and NC-GRU($U_r, U_c$) models significantly outperformed LSTM~\citep{Hochreiter:1997:LSM:1246443.1246450}, GRU~\citep{https://doi.org/10.48550/arxiv.1406.1078}, scoRNN~\citep{kyle17}, and GORU~\citep{8668730} models on the denoise data, see Figure \ref{fig:denoise_task}. Similar to what we observed from the parenthesis task, usage of two orthogonal weights leads to better results. In addition, Table \ref{tab:denoise_task} provides a comparison of attained minimum loss for each of the models.

\begin{table}[ht]
    \centering
    \caption{\textbf{Denoise Task Results:} Minimum attained loss values (the smaller the better). All results are based on our tests.}
    \begin{tabular}{r|P{1.6cm}|P{1.6cm}}
        \toprule
        & \multicolumn{2}{c}{Loss $\times 10^{-2}$ $\downarrow$} \\
        \midrule
        Sequence Length & $T=200$ & $T=400$\\
        \midrule
        LSTM & $10.08$ & $5.427$\\
        GRU & $9.799$ & $5.2338$\\
        scoRNN & $8.258$ & $3.812$\\
        GORU & $4.453$ & $2.29$ \\
        NC-GRU$(U_c)$ (\textbf{ours}) & $1.639$ & $0.878$ \\
        NC-GRU$(U_c, U_r)$ (\textbf{ours}) & $1.633$ & $0.774$ \\
        \bottomrule
    \end{tabular}
    \label{tab:denoise_task}
\end{table}

\begin{figure*}[t]
    \centering
    \subfloat[$T=200$]{{\includegraphics[width=0.45\textwidth]{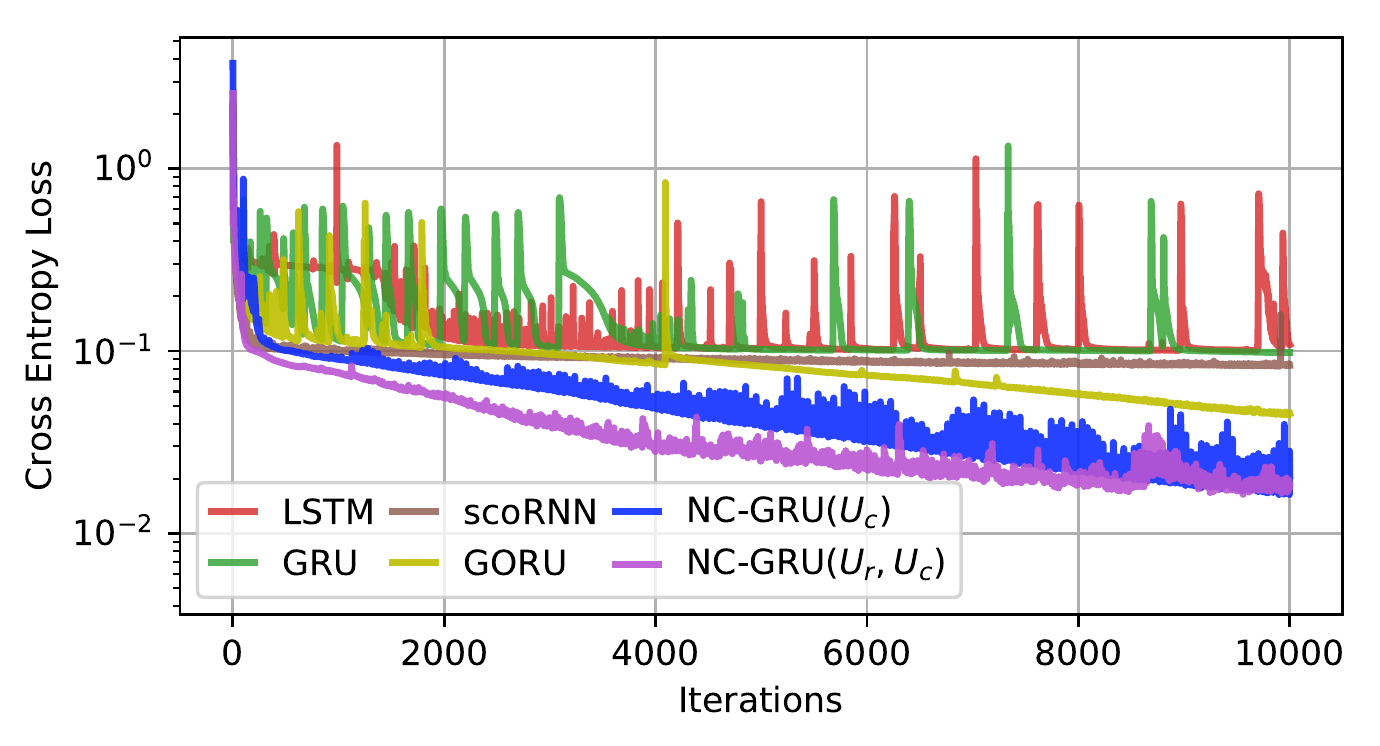} }\label{fig:denoise_task200}}
    \qquad
    \subfloat[$T=400$]{{\includegraphics[width=0.45\textwidth]{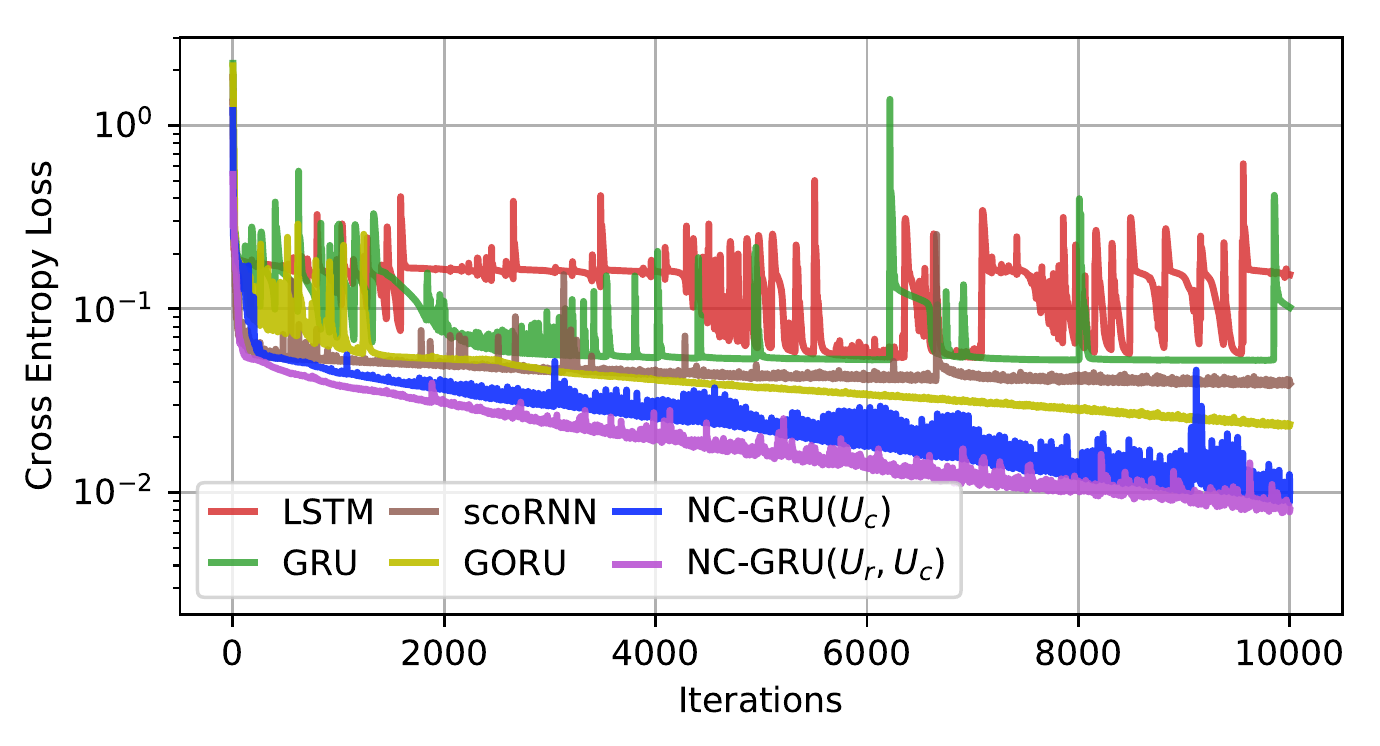} }}
    \caption{\textbf{Denoise Task Results}\label{fig:denoise_task400}}
    \label{fig:denoise_task}
\end{figure*}

\subsection{Adding Problem}\label{addingproblem}

The adding problem is the third synthetic task that we considered and it was proposed in~\citep{hochreiter1997long} for recurrent networks. Our implementation of this problem is a variation of the original problem. The input in the network is a 2-dimensional sequence of length $T$. In the first dimension, we have a sequence of all zeros except for two randomly placed ones, one in the first half of the sequence and one in the second one. In the second dimension, we have a sequence of randomly selected numbers that are chosen uniformly from the interval $[0,1)$. The goal of the adding task is to take the second dimension numbers from the positions that correspond to the ones in the first dimension and then output their sum. The size of the training and testing/evaluation sets are $100,000$ and $10,000$, respectively.

We compared the performance of our NC-GRU models to LSTM~\citep{Hochreiter:1997:LSM:1246443.1246450}, GRU~\citep{https://doi.org/10.48550/arxiv.1406.1078}, scoRNN~\citep{kyle17}, and GORU~\citep{8668730} models.

\emph{Implementation Details:} In this task, we worked with a single layer cell with hidden dimension set to be 68 for LSTM~\citep{Hochreiter:1997:LSM:1246443.1246450}, 70 for GRU~\citep{https://doi.org/10.48550/arxiv.1406.1078}, 190 for scoRNN~\citep{kyle17}, 128 for GORU~\citep{8668730}, and 80 for NC-GRU. The number of negative ones for the $D$ matrix inside scoRNN and NC-GRU was set to 95 and 43, respectively. We used the Adam optimizer~\citep{kingma2014adam} with a learning rate of $10^{-3}$ in all of the experiments, including the training of the $A$ weight in scoRNN~\citep{kyle17} and NC-GRU models. 

All of the models were trained for 10 epochs with a batch size of 50, and evaluation every 100 iterations. Neumann series were applied during the training of NC-GRU models, with reset option implemented every 50 iterations.

\emph{Results:} Figures \ref{fig:adding200} and \ref{fig:adding400} present the performances of all the interested methods on the tests for sequences of length 200 and 400, respectively. In these experiment, our NC-GRU($U_c$) model where the Neumann-Cayley method applied only to the $U_c$ weight showed the best performance out of all the compared models including our NC-GRU($U_r, U_c$) model which produced comparable or marginally better results than others. Minimum attained validation loss values presented in Table \ref{tab:adding_task} for both experiments.

\begin{table}[ht]
    \centering
    \caption{\textbf{Adding Task Results:} Minimum attained loss values (the smaller the better). All results are based on our tests.}\label{tab:adding_task}
    \begin{tabular}{r|P{1.6cm}|P{1.6cm}}
        \toprule
        & \multicolumn{2}{c}{Loss $\times 10^{-5}$ $\downarrow$} \\
        \midrule
        Sequence Length & $T=200$ & $T=400$\\
        \midrule
        LSTM & $2.2829$ & $269.7$\\
        GRU & $1.776$ & $1.2976$\\
        scoRNN & $51.49$ & $16,295.05$\\
        GORU & $2.758$ & $2.3067$ \\
        NC-GRU$(U_c)$ (\textbf{ours}) & $1.022$ & $0.81$ \\
        NC-GRU$(U_c, U_r)$ (\textbf{ours}) & $5.0$ & $4.93$\\
        \bottomrule
    \end{tabular}
\end{table}

\begin{figure*}
    \centering
    \subfloat[$T=200$]{{\includegraphics[width=0.45\textwidth]{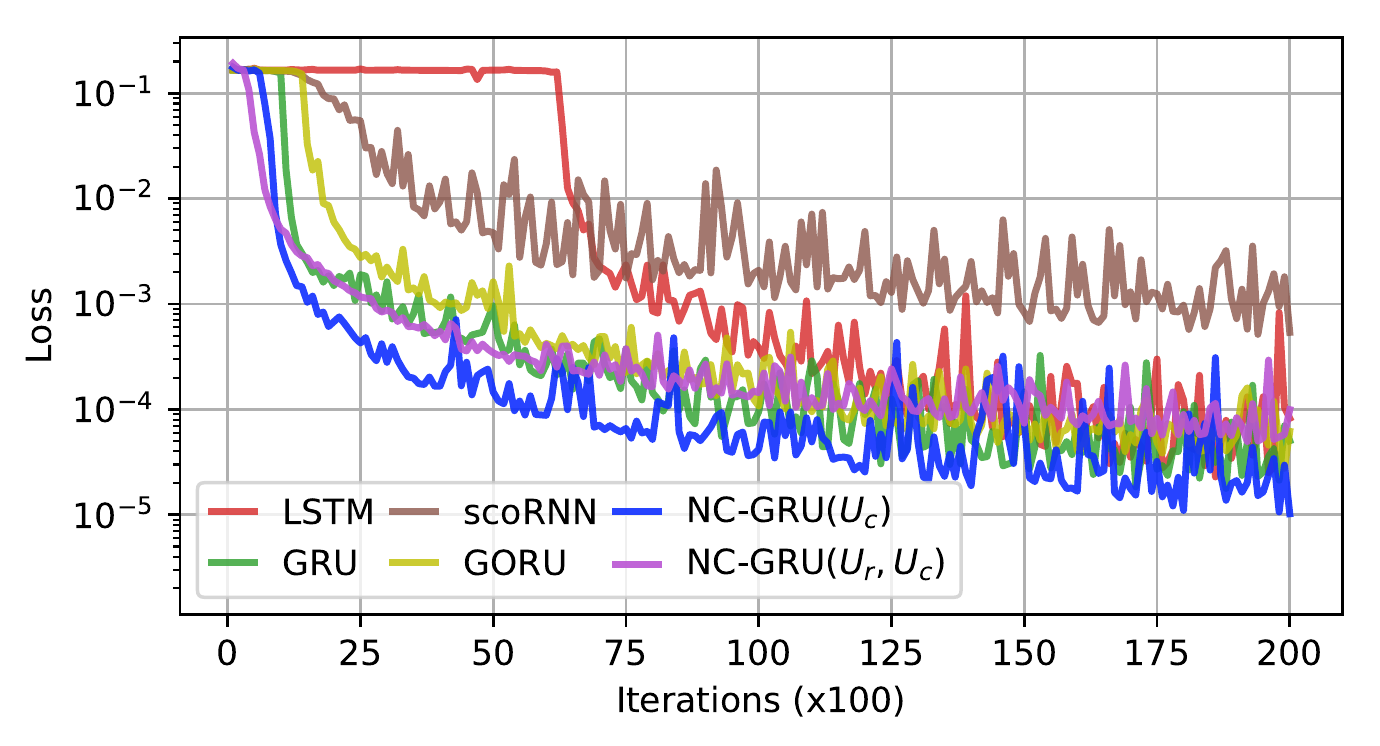} }\label{fig:adding200}}
    \qquad
    \subfloat[$T=400$]{{\includegraphics[width=0.45\textwidth]{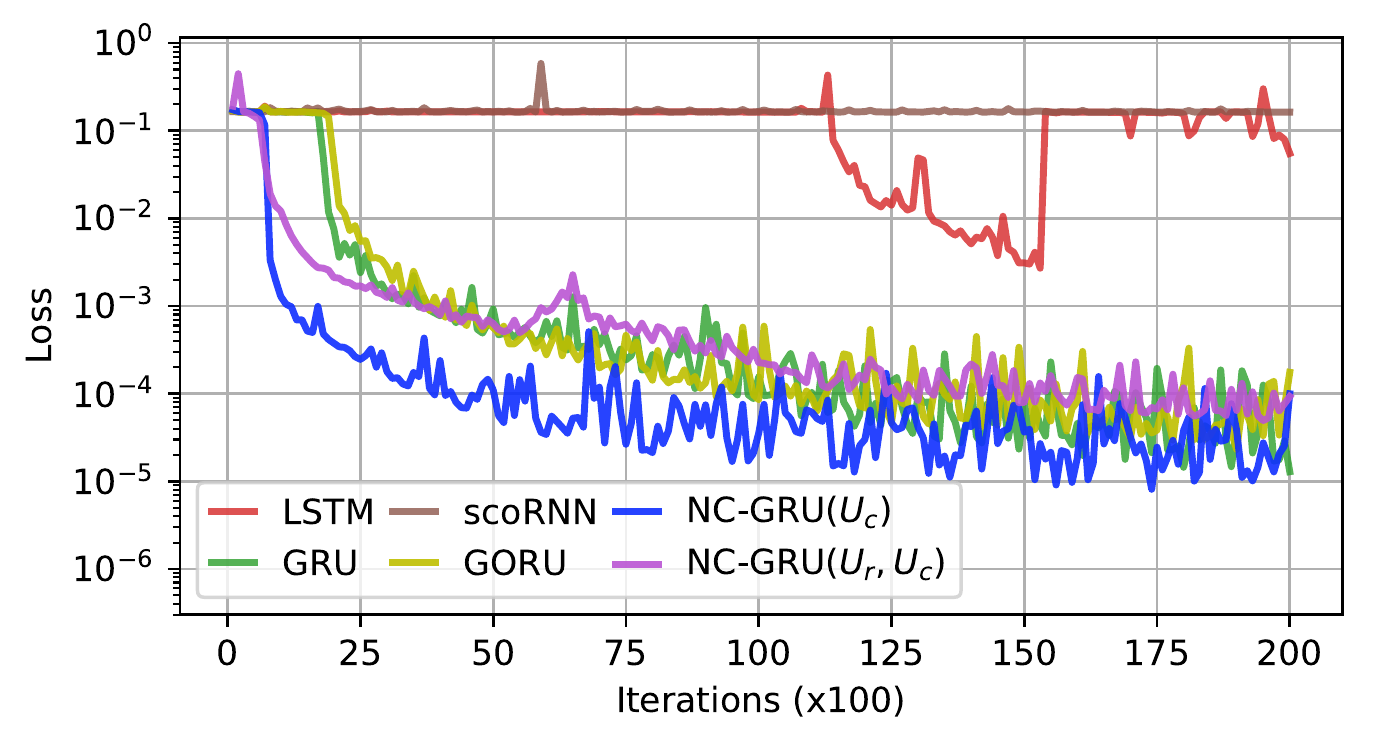} }\label{fig:adding400}}
    \caption{\textbf{Adding Problem Results}\label{fig:adding}}
\end{figure*}

\subsection{Copying Problem}\label{copyingproblem}

The copying problem was proposed in~\citep{hochreiter1997long} as a synthetic task for testing Recurrent Neural Networks. The setup of this problem consists of a string of 10 random digits that are sampled uniformly from the integers 1 through 8 and then fed into the recurrent model. These 10 digits are followed by a string of $T$ zeros and a digit 9 which marks the start of a string of 9 zeros. Therefore, the total length of the fed string is $T+20$. The objective of the task is to output the initial sequence of 10 random digits beginning at the marker location, copying the first ten elements of the sequence in order. For the evaluation of the model, the cross-entropy loss is used with an expected cross-entropy baseline of $\frac{10\log{(8)}}{T+20}$ which represents the random selection of digits 1-8 after the 9.

Similar to the other 3 synthetic experiments, we have compared the performance of our NC-GRU models to LSTM~\citep{Hochreiter:1997:LSM:1246443.1246450}, GRU~\citep{https://doi.org/10.48550/arxiv.1406.1078}, scoRNN~\citep{kyle17}, and GORU~\citep{8668730} models.

\emph{Implementation Details:} We carried out our experiments on a single layer models with hidden sizes for each model to match the number of trainable parameters were 68 for LSTM~\citep{Hochreiter:1997:LSM:1246443.1246450}, 78 for GRU~\citep{https://doi.org/10.48550/arxiv.1406.1078}, 190 for scoRNN~\citep{kyle17}, 100 for GORU~\citep{8668730}, and 96 for NC-GRU. All the models were trained using Adam~\citep{kingma2014adam} optimizer with the learning rate of $10^{-3}$ except $A$ matrix in scoRNN~\citep{kyle17} and NC-GRU models where Adam~\citep{kingma2014adam} optimizer learning rate was set to $10^{-4}$. The number of training iterations and a batch size were set to 10,000 and 50, respectively. 

The number of negative ones for scaling matrix $D$ was set to 95 and 80 for scoRNN~\citep{kyle17} and NC-GRU models respectively. All of the models were evaluated every 50 iterations. While training the model using the NC-GRU models, we applied the Neumann series in a similar way as in the other experiments with the reset option at every 20 iterations.

\emph{Results:} Figures \ref{fig:copying1000} and \ref{fig:copying2000} present the performance of aforementioned models for the Copying Problem with string sizes of $T=1,000$ and $T=2,000$, respectively. Moreover Table \ref{tab:copying_task} shows minimum validation loss values. Both NC-GRU models perform on the same level while outperforming LSTM~\citep{Hochreiter:1997:LSM:1246443.1246450}, GRU~\citep{https://doi.org/10.48550/arxiv.1406.1078}, and GORU~\citep{8668730} models by a noticeable margin. However, for this task scoRNN~\citep{kyle17} performed better than all of the other models.

\begin{table}[ht]
    \centering
    \caption{\textbf{Copying Task Results:} Minimum attained loss values (the smaller the better). All results are based on our tests.}\label{tab:copying_task}
    \begin{tabular}{r|P{1.6cm}|P{1.6cm}}
        \toprule
        & \multicolumn{2}{c}{Loss $\times 10^{-2}$ $\downarrow$} \\
        \midrule
        Sequence Length & $T=1,000$ & $T=2,000$\\
        \midrule
        LSTM & $2.040$ & $1.03$\\
        GRU & $2.039$ & $0.89$\\
        scoRNN & $0.002$ & $0.03$\\
        GORU & $1.488$ & $0.718$ \\
        NC-GRU$(U_c)$ (\textbf{ours}) & $0.904$ & $0.44$ \\
        NC-GRU$(U_c, U_r)$ (\textbf{ours}) & $0.882$ & $0.496$\\
        \bottomrule
    \end{tabular}
    
\end{table}

\begin{figure*}
    \centering
    \subfloat[$T=1,000$]{{\includegraphics[width=0.45\textwidth]{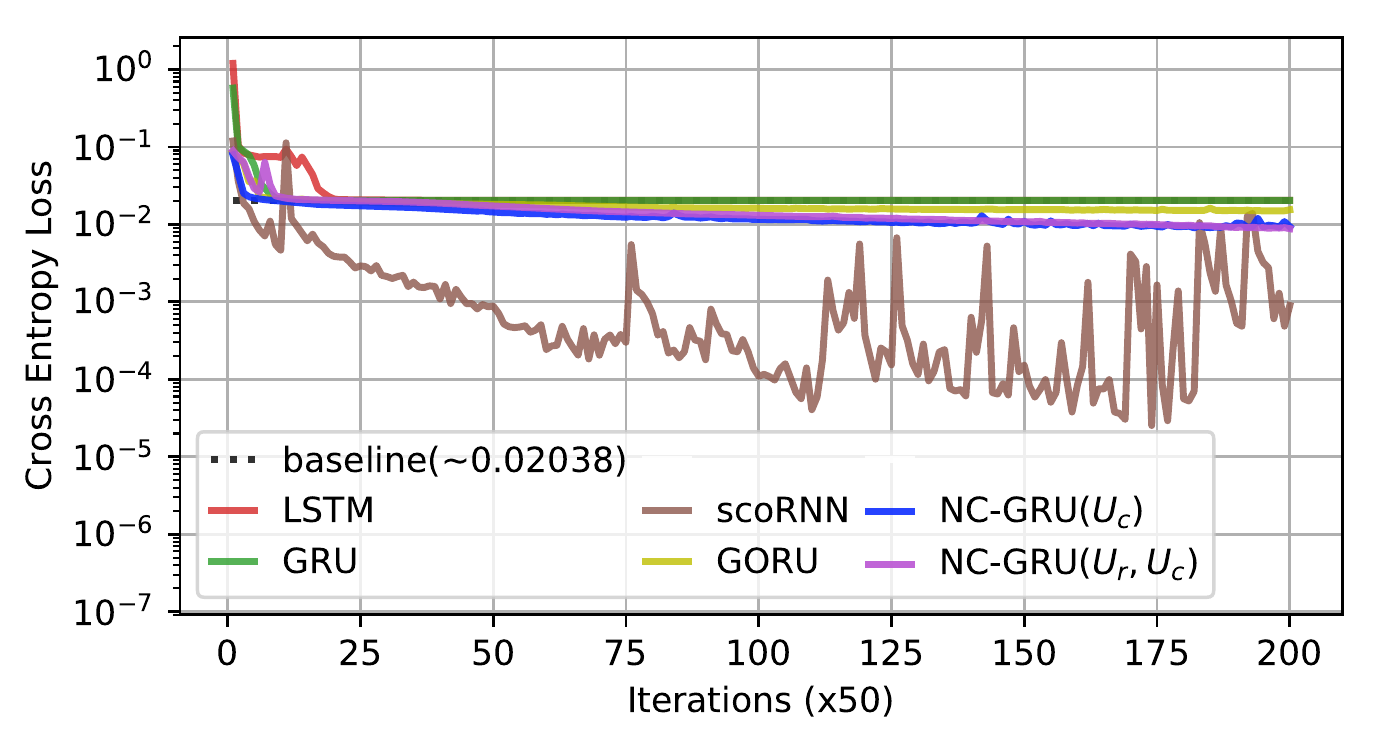} }\label{fig:copying1000}}
    \qquad
    \subfloat[$T=2,000$]{{\includegraphics[width=0.45\textwidth]{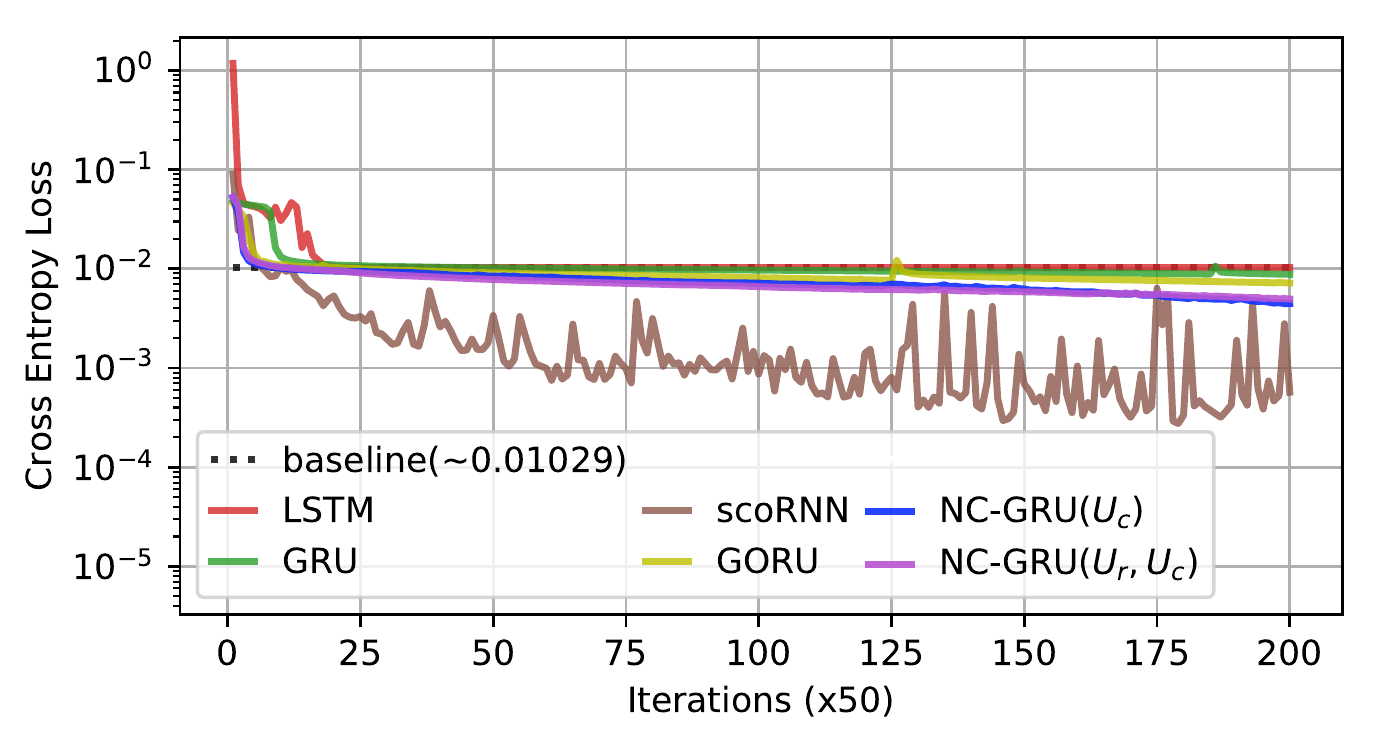} }\label{fig:copying2000}}
    \caption{\textbf{Copying Problem Results}}
    \label{fig:copying}
\end{figure*}

Language modeling is one of many natural language processing tasks. It is the development of probabilistic models that are capable of predicting the next character or word in a sequence using information that has preceded it. We have conducted two experiments on the language modeling with the character and word level tasks for the Penn Treebank (PTB) dataset~\citep{10.5555/972470.972475} both of which were based on experiments and codes from the AWD-LSTM model~\citep{merity2018regularizing}. 

\subsection{Character Level Penn Treebank}\label{cptb}
For this task, models were tested on their suitability for language modeling tasks using the character level Penn Treebank dataset~\citep{marcus-etal-1993-building}. This dataset is a collection of English-language Wall Street Journal articles. The dataset consists of a vocabulary of 10,000 words with other words replaced as $<unk>$, resulting in approximately 6 million characters that are divided into 5.1 million, 400 thousand, and 450 thousand character sets for training, validation, and testing, respectively with a character alphabet size of 50. The goal of the character-level Language Modeling task is to predict the next character given the preceding sequence of characters.

We compared the performance of our NC-GRU model to LSTM~\citep{Hochreiter:1997:LSM:1246443.1246450}, GRU~\citep{https://doi.org/10.48550/arxiv.1406.1078}, and scoRNN~\citep{kyle17} models.

\emph{Implementation Details:} 
For this task, we have considered a three layer models, where hidden dimensions for NC-GRU($U_c$) model were set to (430, 1000, 430). All the dropout coefficient were set to $0.15$. The learning rate of $5\times10^{-4}$ together with the Adam~\citep{kingma2014adam} optimizer were used to train whole model including the matrix $A$. The number of negative ones for matrix $D$ for every layer was set to $50$.

LSTM~\citep{Hochreiter:1997:LSM:1246443.1246450}, GRU~\citep{https://doi.org/10.48550/arxiv.1406.1078}, and scoRNN~\citep{kyle17} model were trained in the similar manner with hidden dimensions being (350, 880, 350), (415, 950, 415), and (500, 2000, 500), respectively. Batch size for all the experiments was set to be $32$ and the Back Propagation Through Time (bptt) window of 100 was used for all the models.

\emph{Results:} We evaluate all models using the bits-per-character (bpc) metric. As shown in Table \ref{tab:ptbchar_task}, NC-GRU showed a better performance than scoRNN, GRU, and even LSTM models.

\begin{table}[ht]
    \centering
    \caption{\textbf{Character Level PTB Results:} Evaluated bits-per-character (bpc) for every model (the smaller the better). All results are based on our tests.}\label{tab:ptbchar_task}
    \begin{tabular}{r|P{3cm}}
        \toprule
        & bpc $\downarrow$ \\
        \midrule
        LSTM & $1.4$ \\
        GRU & $1.455$ \\
        scoRNN & $1.6$ \\
        NC-GRU($U_c$) (\textbf{ours}) & $1.385$ \\
        \bottomrule
    \end{tabular}
\end{table}

\subsection{Word Level Penn Treebank}\label{wptb}
Finally, we also tested our proposed Neumann-Cayley method on the word level Penn Treebank dataset ~\citep{10.5555/972470.972475}. The dataset takes the same underlying corpus as the character level task, but with tokens representing words instead of characters. This results in a smaller dataset with a larger vocabulary size, with 888 thousand, 70 thousand, and 79 thousand words as training, validation, and testing sets, and a vocabulary of 10 thousand words.

\emph{Implementation Details:}
We trained NC-GRU with three layers with dimensions (400, 1150, 400). We have a learning rate set to $5\times 10^{-4}$ for both the $A$ matrix and the rest of the model, optimized using Adam~\citep{kingma2014adam}. The dropout after the NC-GRU cell has a coefficient of 0.4, the embedding layer dropout has a coefficient of 0.4, and the output dropout has a coefficient of 0.25. The number of negative ones for matrix $D$ for each layer was 50.

\emph{Results:} Results were evaluated using the perplexity (PPL) metric and are shown in table \ref{tab:ptbw_task}. We show improved performance over both baseline GRU and LSTM models.

\begin{table}[ht]
    \centering
    \caption{\textbf{Word Level PTB Results:} Evaluated perplexity (PPL) for every model (the smaller the better).  $^*$ - result obtained from our experiments; $^{\dagger}$ - result quoted from~\citep{https://doi.org/10.48550/arxiv.1803.01271}}
    \label{tab:ptbw_task}
    \begin{tabular}{r|P{3cm}}
        \toprule
        & PPL $\downarrow$ \\
        \midrule
        LSTM & $78.93^{\dagger}$ \\
        GRU & $92.48^{\dagger}$ ($80.73^{*}$) \\
        scoRNN & $123.90^{*}$ \\
        NC-GRU($U_c$) (\textbf{ours}) & $77.00$ \\
        \bottomrule
    \end{tabular}
\end{table}

\section{Ablation Studies}\label{ablation_study}

In this section, we consider several ablation studies that help us justify the use of Neumann series, orthogonal matrices, as well as scaled-Cayley transforms.

\subsection{Neumann Series Method vs Inverse}\label{as:inv_vs_neumann}
In this experiment we study the sharpness of approximating $\left(I+A^{(k-1)}\right)^{-1}$ with Neumann series in the NC-GRU($U_c$) model on the Parenthesis task, see \ref{parenthesisproblem} for implementation details and description of NC-GRU($U_c$) model. We consider using Neumann series approximation of order 1, 2, and 3 and compare them to the method of Least Squares for taking a matrix inverse which is one of the widely used methods from Deep Learning libraries such as TensorFlow, PyTorch, and NumPy.

\begin{figure}[ht]
    \centering 
    \includegraphics[width=0.45\textwidth]{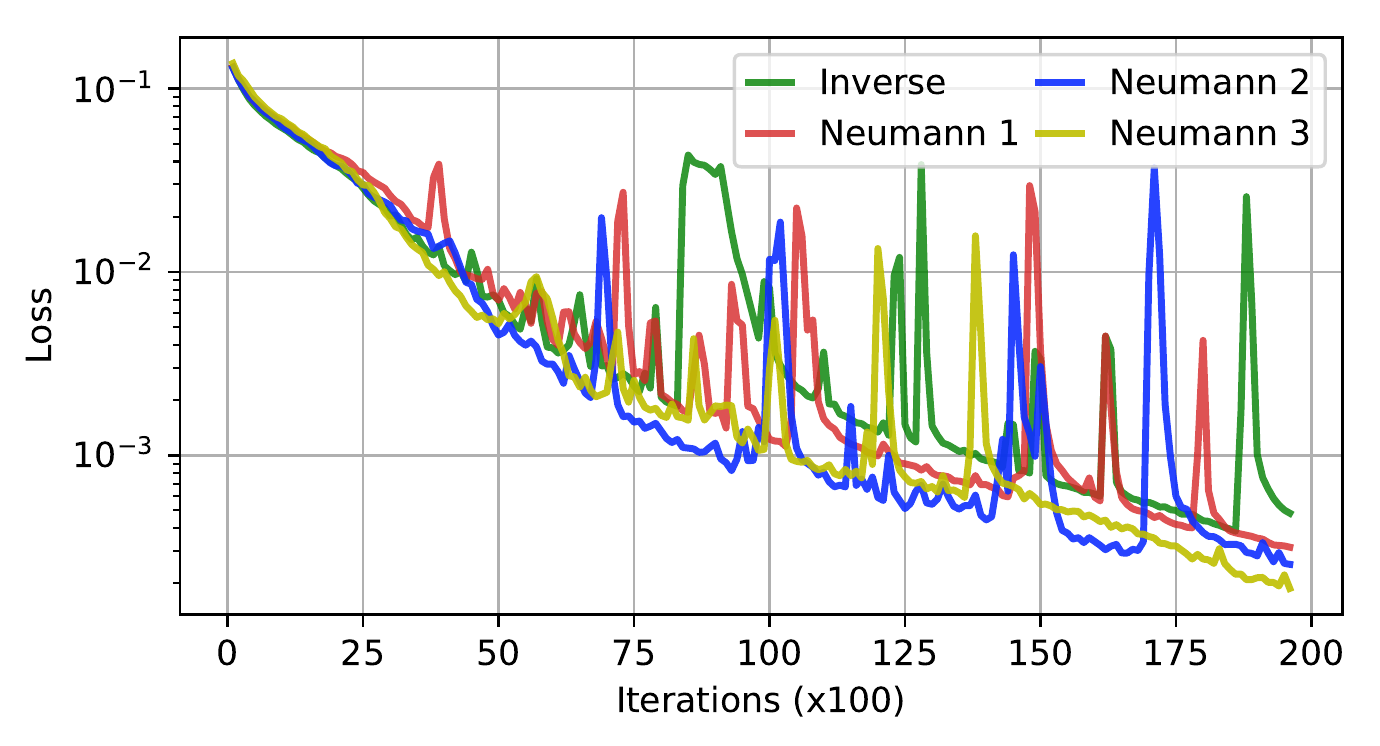}
    \caption{\textbf{Inverse vs Neumann series:} In this graphs, \texttt{Inverse} represents NC-GRU($U_c$) model with $\left(I+A^{(k-1)}\right)^{-1}$ computed using the Least Squares method and \texttt{Neumann i} represents NC-GRU($U_c$) model with $\left(I+A^{(k-1)}\right)^{-1}$ computed using the Neumann series method of order $i$.}
    \label{fig:inv_vs_neumann}
\end{figure}

Our experiments showed that the Neumann series approximation method achieves better results in comparison to the classical Least Squares method for finding matrix inverse. Particularly, Figure \ref{fig:inv_vs_neumann} shows that the Neumann series method of order 2 performs marginally better than order 1 and order 3 Neumann series methods and significantly better than the Least Squares method. 

\begin{figure}[ht]
    \centering 
    \includegraphics[width=0.45\textwidth]{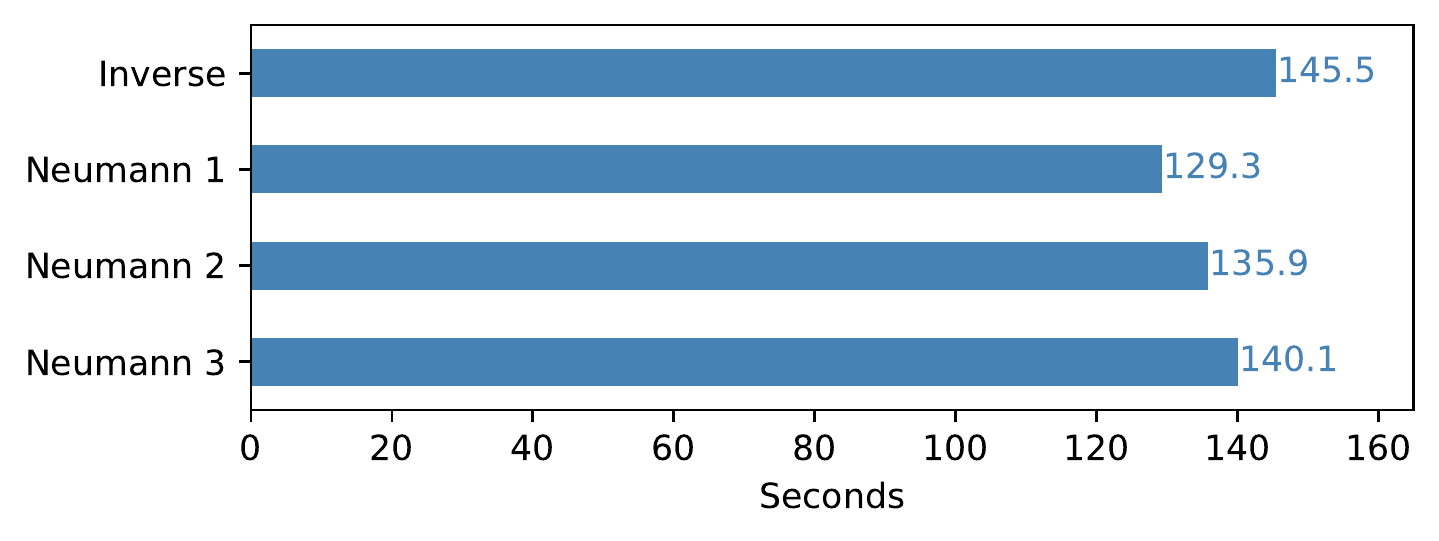}
    \caption{\textbf{Computational time comparisons:} The amount of time (in seconds) it takes to train one epoch of NC-GRU($U_c$) model on Character Level PTB dataset using a single NVIDIA$^{\circledR}$ Tesla$^{\circledR}$ V100 GPU.}
    \label{fig:inv_vs_neumann_time_ptb}
\end{figure}

Additionally, we have compared the time it takes to train our model using mentioned methods. Figure \ref{fig:inv_vs_neumann_time_ptb} shows the time, it takes to train one epoch of NC-GRU($U_c$) model on a Character Level PTB dataset on a single NVIDIA$^{\circledR}$ Tesla$^{\circledR}$ V100 GPU using Inverse (Least Square), Neumann 1, Neumann 2, and Neumann 3 methods.

The observed behavior appears to be quite general and we have conducted all of the experiments in section \ref{experiments} using the second order Neumann series method.

\subsection{Orthogonality in the Hidden Weights of GRU}\label{as:ortho_in_Uu_Ur}
For our second ablation study, we have studied the effect of Neumann-Cayley transformation orthogonal weights applied to the hidden units inside the GRU cell (\ref{model:GRU}). We have considered three models. The first model only had $U_c$ weight replaced with orthogonal weight preserved by the Neumann-Cayley method, we previously called such a model NC-GRU($U_c$). The second model, NC-GRU($U_r, U_c$) had two weights $U_c$ and $U_r$ replaced with Neumann-Cayley transformation orthogonal weights, and finally, the third model, NC-GRU($U_r, U_u, U_c$) had all three weights $U_r$, $U_u$, and $U_c$ replaced.

The results are shown in Figure \ref{fig:one_vs_all}. We see that implementing one or two orthogonal weight models, NC-GRU($U_c$) or NC-GRU($U_r, U_c$), would have the most benefits, while the three orthogonal weight model NC-GRU($U_r, U_u, U_c$) does not perform as well. This can also be seen in experiments in section \ref{experiments} where both one and two orthogonal weight models are used.

\begin{figure}[ht]
    \centering
    \includegraphics[width=0.45\textwidth]{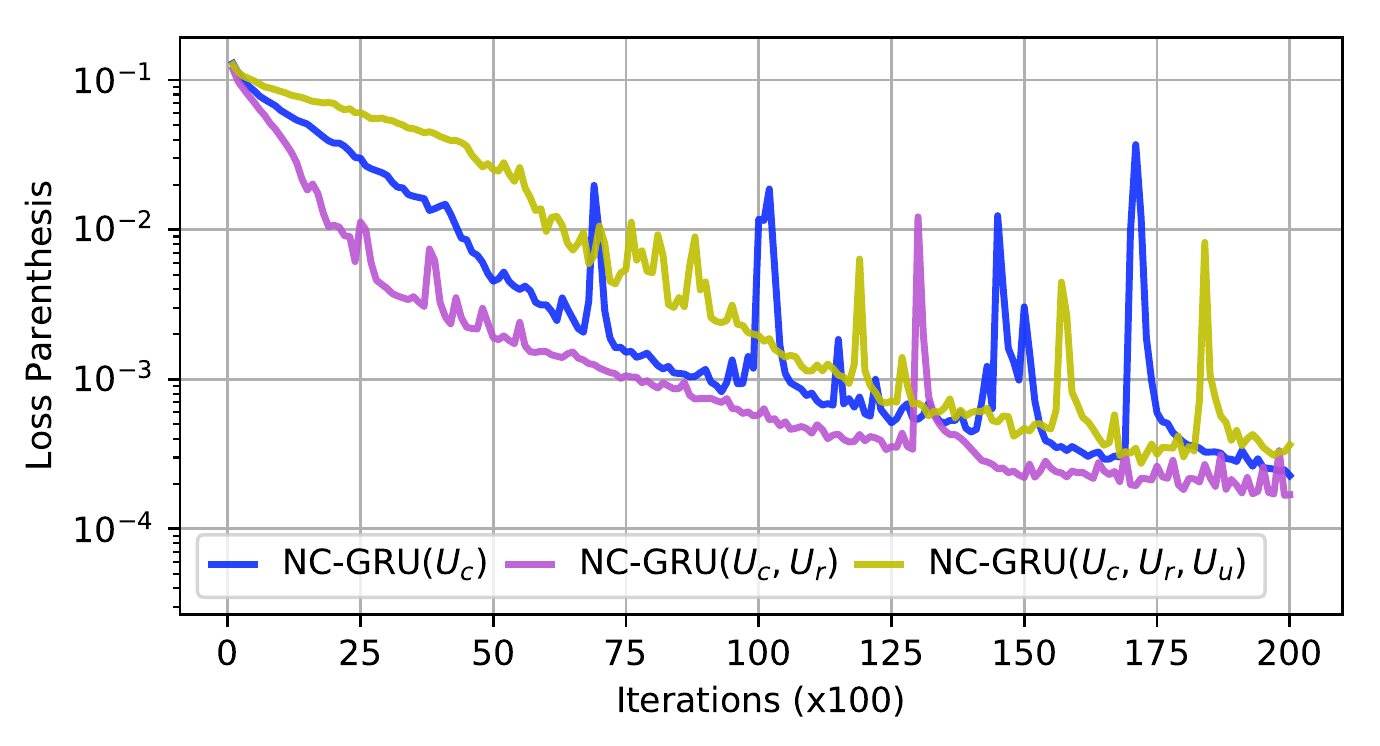}
    \caption{\textbf{Orthogonal matrices for different combinations of hidden weights:} NC-GRU one hidden parameter orthogonal vs three hidden parameters orthogonal}
    \label{fig:one_vs_all}
\end{figure}

\subsection{Necessary Condition for the Neumann Series Method}\label{as:norms}
As it was mentioned in \ref{sec:backprop}, the assumption that allowed us to use the Neumann series is $\norm{\left(I+A^{(k-1)}\right)^{-1}\delta A^{(k)}}<1$ where $\norm{\cdot}$ satisfies $\norm{AB}\leq \norm{A}\norm{B}$ for some matrices $A$ and $B$ of appropriate dimensions.

\begin{figure}[ht]
    \centering 
    \includegraphics[width=0.45\textwidth]{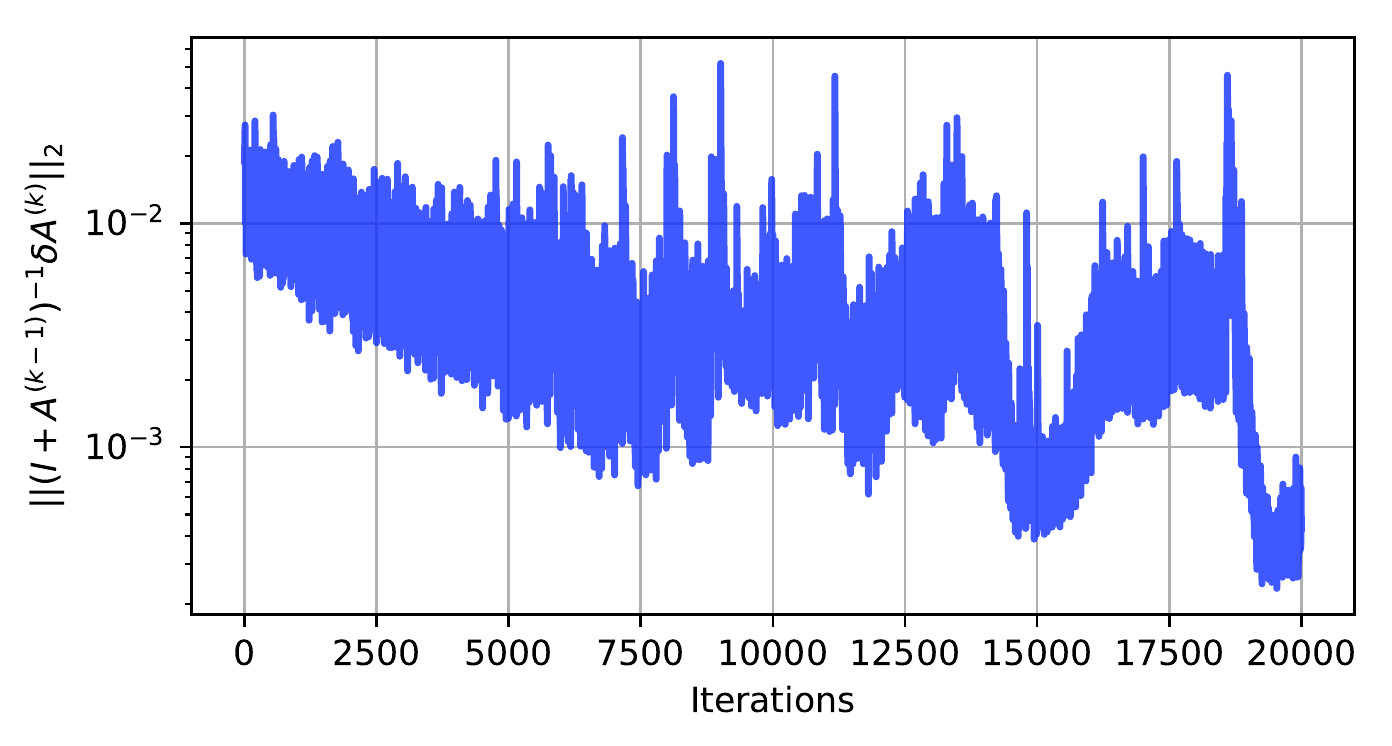}
    \caption{\textbf{Norm condition for the Neumann approximation:} Values of $\norm{\left(I+A^{(k-1)}\right)^{-1}\delta A^{(k)}}_2$ using NC-GRU($U_c$) model with second-order Neumann series approximation method.}
    \label{fig:neumann_norms}
\end{figure}

For this experiment, we chose the Spectral Norm, i.e. $L_2$-norm, and recorded values $\norm{\left(I+A^{(k-1)}\right)^{-1}\delta A^{(k)}}_2$ during training the NC-GRU($U_c$) model on the Parenthesis task with the second-order Neumann series approximation. We observed that values of the norm were changing during the training however they are well below 1 and satisfy the necessary condition as depicted in Figure \ref{fig:neumann_norms}.

\section{Conclusion}\label{conclusion}
In this paper, we have presented a thorough analysis of the Gated Recurrent Unit (GRU) model's gradients. Based on this analysis, we introduced the Neumann-Cayley Gated Recurrent Unit model, which we called NC-GRU. Our model incorporates orthogonal weights in the hidden states of the GRU model which are trained using a newly proposed method of Neumann-Cayley transformation for maintaining the desired orthogonality in those weights. We have conducted a series of experiments that demonstrate the superiority of our proposed method outperforming GRU, LSTM, scoRNN, and GORU moduls on different synthetic and real-world tasks. Moreover, we conducted several ablation studies that empirically confirmed our theoretical results.

\section{Acknowledgment}
We would like to thank the University of Kentucky Center for Computational Sciences and Information Technology Services Research Computing for their support and use of the Lipscomb Compute Cluster and associated research computing resources. 
This research was supported in part by NSF under grants DMS-1821144, DMS-2053284 and DMS-2151802, and University of Kentucky Start-up fund.

\bibliographystyle{apalike}
\bibliography{bibliography}

\onecolumn

\section{Supplementary Materials}\label{sm:sec}

In this section, we present the proofs for the theorem and corollaries stated in section \ref{theory}. 

\setcounter{theorem}{0}
\begin{theorem}
    Let $h_{t-1}$ and $h_{t}$ be two consecutive hidden states from the GRU model stated in (\ref{model:GRU}). Then
    \begin{equation}
        \norm{\dfrac{\partial h_{t}}{\partial h_{t-1}}}_2 \leq \alpha + \beta\norm{U_c}_2
    \end{equation}
    where
    \begin{equation}
        \begin{aligned}
            \alpha &= \delta_u\left( \max_i\left\{[h_{t-1}]_i\right\} + \max_i\left\{[c_t]_i\right\}\right)\norm{U_u}_2 + \max_i\left\{(1-[u_t]_i)\right\}
        \end{aligned}
    \end{equation}
    and
    \begin{equation}
        \begin{aligned}
            \beta &= \max_i\left\{[u_t]_i\right\}\left(\delta_r \norm{U_r}_2 \max_i \left\{[h_{t-1}]_i\right\}+ \max_i \left\{[r_t]_i\right\} \right),
        \end{aligned}
    \end{equation}
    with constants $\delta_u$ and $\delta_r$ be as follows:
    \begin{equation}
        \delta_u = \max_i \left\{\left[u_t\right]_i \left(1-\left[u_t\right]_i\right)\right\}
    \end{equation}
    and
    \begin{equation}
        \delta_r = \max_i \left\{\left[r_t\right]_i \left(1-\left[r_t\right]_i\right)\right\}.
    \end{equation}
\end{theorem}

\begin{proof}
    Since $h_t$ depends on $u_t$ and $c_t$ (which is also depends on $r_t$), we will start by finding $\dfrac{\partial r_t}{\partial h_{t-1}}$, and $\dfrac{\partial u_t}{\partial h_{t-1}}$, and $\dfrac{\partial c_t}{\partial h_{t-1}}$ as well as corresponding bounds of these gradients. Recall that
    \begin{equation}
        r_t = \sigma\left(W_rx_t+U_rh_{t-1}+b_r\right)
    \end{equation}
    then 
    \begin{equation}\label{eq:dr_t/dh_t-1}
        \dfrac{\partial r_t}{\partial h_{t-1}}=diag\left(\sigma'\left(W_rx_t+U_rh_{t-1}+b_r\right)\right)U_r=diag\left(r_t \odot (1-r_t) \right)U_r
    \end{equation}
    since $\sigma'(x)=\sigma(x)\left(1-\sigma(x)\right)$. Moreover, we can bound $\dfrac{\partial r_t}{\partial h_{t-1}}$ in the following way
    \begin{equation}\label{rt_bound:1}
        \norm{\dfrac{\partial r_t}{\partial h_{t-1}}}_2\leq \delta_r\norm{U_r}_2
    \end{equation}
    with
    \begin{equation}\label{cr_const:1}
        \delta_r := \max_i \left\{\left[r_t\right]_i \left(1-\left[r_t\right]_i\right)\right\}.
    \end{equation}

    The constant $\delta_r$ is defined to be the largest entry of the vector $r_t \odot (1-r_t)$ and it is bounded by $\dfrac{1}{4}$. Similarly,
    \begin{equation}
        \dfrac{\partial u_t}{\partial h_{t-1}}=diag\left(\sigma'\left(W_ux_t+U_uh_{t-1}+b_u\right)\right)U_u=diag\left(u_t \odot (1-u_t) \right)U_u,
    \end{equation}
    and
    \begin{equation}\label{ut_bound:1}
        \norm{\dfrac{\partial u_t}{\partial h_{t-1}}}_2\leq \delta_u\norm{U_u}_2
    \end{equation}
    with
    \begin{equation}\label{cu_const:1}
        \delta_u := \max_i \left\{\left[u_t\right]_i \left(1-\left[u_t\right]_i\right)\right\}.
    \end{equation}

    The definition of the vector $c_t$ is
    \begin{equation}
        c_t = \Phi\left(W_cx_t+U_c\left(r_t \odot h_{t-1}\right)+b_c \right)
    \end{equation}
    and
    \begin{equation}
        \dfrac{\partial c_t}{\partial h_{t-1}}= diag\left(\Phi'\left(W_cx_t+U_c(r_t \odot h_{t-1})+b_c \right) \right) U_c \left( diag\left(h_{t-1}\right)\dfrac{\partial r_t}{\partial h_{t-1}} +diag\left(r_t\right) \right),
    \end{equation}
    with $\Phi'$ applied entrywise, and 
    \begin{align}
        \norm{\dfrac{\partial c_t}{\partial h_{t-1}}}_2 &\leq \max_i \left\{\left[\Phi'\left(W_cx_t+U_c\left(r_t \odot h_{t-1}\right)+b_c\right)\right]_i\right\} \left(\norm{\dfrac{\partial r_t}{\partial h_{t-1}}}_2 \max_i \left\{[h_{t-1}]_i\right\}+\max_i \left\{[r_t]_i\right\} \right) \norm{U_c}_2 \\
        &\leq \left(\delta_r \norm{U_r}_2\max_i \left\{[h_{t-1}]_i\right\}+\max_i \left\{[r_t]_i\right\}  \right)\norm{U_c}_2\label{ct_bound:1}
    \end{align}
    since $\Phi'$ is bounded by 1.
    
    Finally,
    \begin{equation}
        h_t = \left(1 - u_t\right) \odot h_{t-1} + u_t \odot c_t
    \end{equation}
    with 
    \begin{equation}
        \dfrac{\partial h_t}{\partial h_{t-1}} = -diag\left(h_{t-1}\right) \dfrac{\partial u_t}{\partial h_{t-1}} + diag\left(1-u_t\right) + diag\left(c_t\right) \dfrac{\partial u_t}{\partial h_{t-1}} + diag\left(u_t\right) \dfrac{\partial c_t}{\partial h_{t-1}}
    \end{equation}
    and
    \begin{equation}
        \norm{\dfrac{\partial h_t}{\partial h_{t-1}}}_2 
    \leq \max_i\left\{[h_{t-1}]_i\right\} \norm{\dfrac{\partial u_t}{\partial h_{t-1}}}_2 + \max_i\left\{(1-[u_t]_i)\right\}+\max_i\left\{[c_t]_i\right\} \norm{\dfrac{\partial u_t}{\partial h_{t-1}}}_2+\max_i\left\{[u_t]_i\right\}\norm{\dfrac{\partial c_t}{\partial h_{t-1}}}_2.
    \end{equation}
    
    Furthermore, using (\ref{rt_bound:1}), (\ref{ut_bound:1}), and (\ref{ct_bound:1}), we get
    \begin{align}
        \norm{\dfrac{\partial h_t}{\partial h_{t-1}}}_2 &\leq \delta_u\left( \max_i\left\{[h_{t-1}]_i\right\} + \max_i\left\{[c_t]_i\right\}\right)\norm{U_u}_2 + \max_i\left\{(1-[u_t]_i)\right\}\nonumber\\
        &\quad \quad + \max_i\left\{[u_t]_i\right\}\left(\delta_r \norm{U_r}_2\max_i \left\{[h_{t-1}]_i\right\}+ \max_i \left\{[r_t]_i\right\} \right)\norm{U_c}_2\\
        &=: \alpha+\beta\norm{U_c}_2
    \end{align}
    where
    \begin{equation}
        \alpha = \delta_u\left( \max_i\left\{[h_{t-1}]_i\right\} + \max_i\left\{[c_t]_i\right\}\right)\norm{U_u}_2 + \max_i\left\{(1-[u_t]_i)\right\}
    \end{equation}
    and
    \begin{equation}
        \beta = \max_i\left\{[u_t]_i\right\}\left(\delta_r \norm{U_r}_2 \max_i \left\{[h_{t-1}]_i\right\}+ \max_i \left\{[r_t]_i\right\} \right).
    \end{equation}
\end{proof}

\setcounter{corollary}{0}
\begin{corollary}
    For the hyperbolic tangent activation function in (\ref{model:GRU}) (i.e. $\Phi=\,$\verb|tanh|), we have $\delta_u, \delta_r \le \frac{1}{4}$, $[h_t]_i \le 1$ for any $i$ and $t$ as well as 
    \begin{equation}
        \alpha \le \dfrac{1}{2}\norm{U_u}_2+1 \quad\text{and}\quad \beta\leq \dfrac{1}{4}\norm{U_r}_2+1.
    \end{equation}
\end{corollary}

\begin{proof}
    The function $\sigma'(x)=\sigma(x)\left(1-\sigma(x) \right)$ is bounded above by $\frac{1}{4}$, thus both $\delta_u, \delta_r \leq \frac{1}{4}$.
    
    Now, to show that $[h_t]_i \le 1$ for any $i$ and $t$, we first need to note that $h_0$ is initialized to zero (i.e. $[h_0]_i=0$ for all $i$) and $0\leq [u_t]_i, [c_t]_i \leq 1$ for all $i$ and $t$ from the definition of GRU cell in (\ref{model:GRU}). Then for any fixed $i$ 
    \begin{align}
        [h_0]_i &= 0 \\
        [h_1]_i &= [1-u_1]_i\cdot [h_0]_i + [u_1]_i\cdot[c_1]_i \\
        & = [u_1]_i\cdot[c_1]_i \leq 1 
    \end{align}
    Furthermore, if we assume that $[h_{\tau}]_i\leq 1$ for some $\tau \geq 1$, then
    \begin{align}
        [h_{\tau+1}]_i &= [1-u_{\tau+1}]_i\cdot [h_{\tau}]_i + [u_{\tau+1}]_i\cdot[c_{\tau+1}]_i \\
        & \leq [1-u_{\tau+1}]_i\cdot 1 + [u_{\tau+1}]_i\cdot 1 \\
        & = [1-u_{\tau+1}]_i + [u_{\tau+1}]_i=1.
    \end{align}
    
    Thus, by induction we can conclude that $[h_t]_i\leq 1$ for any $i$ and $t$. Finally, using these obtained bounds, we can bound constants $\alpha$ and $\beta$ as follows
    \begin{align}
        \alpha &= \delta_u\left( \max_i\left\{[h_{t-1}]_i\right\} + \max_i\left\{[c_t]_i\right\}\right)\norm{U_u}_2 + \max_i\left\{(1-[u_t]_i)\right\}\\
        & \leq \dfrac{1}{4}\cdot\left(1 + 1 \right) \cdot\norm{U_u}_2 + 1 \\
        &= \dfrac{1}{2} \norm{U_u}_2 + 1
    \end{align}
    and
    \begin{align}
        \beta &= \max_i\left\{[u_t]_i\right\}\left(\delta_r \norm{U_r}_2 \max_i \left\{[h_{t-1}]_i\right\}+ \max_i \left\{[r_t]_i\right\} \right) \\
        & \leq 1\cdot \left(\dfrac{1}{4}\cdot\norm{U_r}_2\cdot 1 + 1\right)\\
        & = \dfrac{1}{4}\norm{U_r}_2 + 1.
    \end{align}
    
    Note, that all of these bounds are independent of $i$ and $t$.
    
\end{proof}

The following two corollaries use a notation $x \approx y$ and $x \lesssim y$ to denote that $x$ is approximately equal to $y$ and $x$ is bounded by a quantity that is approximately equal to $y$, respectively.

\begin{corollary}
    When the elements of GRU gates $u_t$ and $r_t$ are nearly either 0 or 1, then constants $\alpha$ and $\beta$ from Theorem \ref{thm:1} satisfy the following inequality:
    \begin{equation}
        \alpha+\beta\lesssim 2.
    \end{equation}
    Moreover if $u_t$ and $r_t$ are nearly either the zero vector or the vector of all ones, then 
    \begin{equation}
        \alpha+\beta\lesssim 1.
    \end{equation}
\end{corollary}

\begin{proof}
    Recall the definition of $\alpha$ and $\beta$
    \begin{equation}
        \alpha = \delta_u\left( \max_i\left\{[h_{t-1}]_i\right\} + \max_i\left\{[c_t]_i\right\}\right)\norm{U_u}_2 + \max_i\left\{(1-[u_t]_i)\right\}
    \end{equation}
    and
    \begin{equation}
        \beta = \max_i\left\{[u_t]_i\right\}\left(\delta_r \norm{U_r}_2 \max_i \left\{[h_{t-1}]_i\right\}+ \max_i \left\{[r_t]_i\right\} \right).
    \end{equation}

    If we assume that $[u_t]_i \approx 0$ and $[u_t]_j \approx 1$ for some $i\neq j$, then $\delta_u \approx 0$ and $\alpha \lesssim 1$. Additionally, if we assume that the elements of $r_t$ are nearly either 0 or 1, then $\delta_r \approx 0$ and $\beta \lesssim 1$. Putting these two inequalities together yields 
    \begin{equation}
        \alpha + \beta \lesssim 2.
    \end{equation}
    
    Now, if we assume that $u_t$ is nearly a zero vector (i.e. $\max_i\left\{[u_t]_i \right\} \approx 0$), then constant $\delta_u \approx 0$, $\alpha \approx 1$, and $\beta \approx 0$.
    On the other hand, if we assume that $u_t$ is nearly the vector of all ones (i.e. $\min_i\left\{[u_t]_i \right\} \approx 1$) then $\delta_u \approx 0$ and $\alpha \approx 0$. Moreover, if we also assume that $r_t$ is nearly a zero vector (i.e. $\max_i\left\{[r_t]_i \right\} \approx 0$) then $\delta_r \approx 0$ and $\beta \approx 0$. However, if we assume that $r_t$ approaches a vector of all ones (i.e. $\min_i\left\{[r_t]_i \right\} \approx 1$), then $\delta_r \approx 0$ but $\beta \approx 1$ for this case. Putting all of these cases together, we obtain
    \begin{equation}
        \alpha + \beta \lesssim 1.
    \end{equation}
    
\end{proof}

\begin{corollary}
    Let $h_{t-1}$ and $h_{t}$ be two consecutive hidden states from the NC-GRU model defined in (\ref{model:NC-GRU}). Then $\norm{U_r}_2=\norm{U_c}_2=1$ and if element of the gates $u_t$ and $r_t$ are nearly either 0 or 1, then the following inequality is satisfied:
    \begin{equation}
        \norm{\dfrac{\partial h_{t}}{\partial h_{t-1}}}_2 \lesssim 2.
    \end{equation}
    Furthermore, if $u_t$ and $r_t$ are nearly either the zero vector or the vector of all ones,
    \begin{equation}
        \norm{\dfrac{\partial h_{t}}{\partial h_{t-1}}}_2 \lesssim 1.
    \end{equation}
\end{corollary}

\begin{proof}
    By the definition of NC-GRU model, weights $U_r$ and $U_c$ are orthogonal which implies $\norm{U_r}_2=\norm{U_c}_2=1$, and by Theorem \ref{thm:1} and Corollary \ref{cor:1}, we conclude 
    \begin{equation}
        \norm{\dfrac{\partial h_{t}}{\partial h_{t-1}}}_2 \lesssim 2
    \end{equation}
    when element of the gates $u_t$ and $r_t$ approach either 0 or 1; and \begin{equation}
        \norm{\dfrac{\partial h_{t}}{\partial h_{t-1}}}_2 \lesssim 1
    \end{equation}
    if $u_t$ and $r_t$ approach either zero vector or vector of all ones.
    
\end{proof}
\vfill
\end{document}